\documentclass[12pt]{article}

\usepackage[]{todonotes}
\usepackage{times}
\usepackage{soul}
\usepackage{url}
\usepackage[hidelinks]{hyperref}
\usepackage[utf8]{inputenc}
\usepackage[small]{caption}
\usepackage{amsmath}
\usepackage{bookmark}
\usepackage{booktabs}
\usepackage{xcolor}
\usepackage{algorithm}
\usepackage[normalem]{ulem} 
\usepackage{amssymb}
\usepackage{comment}
\usepackage{tikz-cd}
\usepackage{graphicx}
\usepackage{multirow}

\usepackage[noend]{algpseudocode}

\usepackage{pgf}
\usepackage{tikz}
\usetikzlibrary{arrows,automata,positioning}
\usepackage{nicefrac}
\usepackage{biblatex}
\bibliography{jaamasbib.bib} 

\usepackage{subcaption} 

\newcommand{\unravel}[1]{\textsc{Unravel}(\mathbf{#1})}
\newcommand{\minsum}{\textsc{MinSum}}
\newcommand{\minmax}{\textsc{MinMax}}
\newcommand{\rank}{\texttt{rank}}
\newcommand{\minsumbound}{\textsc{BoundedMinSum}}
\newcommand{\minmaxbound}{\textsc{BoundedMinMax}}

\newcommand{\I}{\mathcal{I}}
\newcommand{\N}{\mathcal{N}}
\newcommand{\Var}[1]{Var({#1})}

\newcommand{\Dom}{\mathcal{D}}
\newcommand{\cert}{\textbf{c}}
\newcommand{\C}[1]{\mathcal{C}({#1})}
\newcommand{\Crest}[2]{\mathcal{C}_{#1}({#2})}

\DeclareMathOperator*{\argmin}{arg\,min}

\newcommand{\blang}[2]{\textsc{Bool}[{#1}]_{#2}}
\newcommand{\liqlang}[2]{\textsc{Liquid}[{#1}]_{#2}}

\newcommand{\Blang}{\textsc{Bool}}
\newcommand{\Liqlang}{\textsc{Liquid}}
\newcommand{\Prof}{\boldsymbol{B}}

\newcommand{\citet}[1]{\citeauthor{#1}~\cite{#1}}
\newcommand{\rest}{\upharpoonright}
\newcommand{\lev}{\texttt{lev}}

\newcommand{\U}{\mathbf{U}}
\newcommand{\DU}{\mathbf{DU}}
\newcommand{\RU}{\mathbf{RU}}
\newcommand{\DRU}{\mathbf{DRU}}

\newtheorem{theorem}{Theorem}
\newtheorem{lemma}{Lemma}
\newtheorem{proposition}{Proposition}
\newtheorem{definition}{Definition}
\newtheorem{example}{Example}
\newtheorem{proof}{Proof}
\newtheorem{remark}{Remark}

\newcommand{\noterc}[1]{\todo[inline,color=green!40!white,caption=]{Rachael: #1}}

\newsavebox{\1}
\sbox{\1}{
\begin{tabular}{|c|c|}
\hline
$a$ & 1 \\
\hline
  \end{tabular} }
  
  \newsavebox{\2}
\sbox{\2}{
\begin{tabular}{|c|c|}
\hline
$b$ & 1 \\
\hline
  \end{tabular} }
  
    \newsavebox{\3}
\sbox{\3}{
\begin{tabular}{|c|c|}
\hline
$c$ & 0 \\
\hline
  \end{tabular} }

      \newsavebox{\4}
\sbox{\4}{
\begin{tabular}{|c|c|}
\hline
$d$ & 0 \\
\hline
  \end{tabular} }
    
        \newsavebox{\5}
\sbox{\5}{
\begin{tabular}{|c|c|}
\hline
$e$ & 1 \\
\hline
  \end{tabular} }
  
          \newsavebox{\6}
\sbox{\6}{
\begin{tabular}{|c|c|}
\hline
$f$ & 1 \\
\hline
  \end{tabular} }

\title{Unravelling multi-agent ranked delegations
\thanks{This paper revises and extends our previous work presented at IJCAI-2020 \cite{colley2020smart}, which was developed from ideas discussed at the Dagstuhl Seminar 19381 on Application-Oriented Computational Social Choice in September 2019. We are grateful for the feedback received by the anonymous reviewers of IJCAI-2020 and JAAMAS, as well as the audience of MPREF-2020 and the COMSOC video seminar. Some of the work in this paper was performed while the third author was affiliated with the Institute for Logic, Language and Computation (ILLC) at the University of Amsterdam.}
}

\author
{Rachael Colley,$^{1}$ Umberto Grandi,$^{1}$ Arianna Novaro$^{2}$\\
\normalsize{
$\{$rachael.colley, umberto.grandi$\}@$irit.fr, arianna.novaro@univ-paris1.fr}\\
\\
\normalsize{$^{1}$ Institut de Recherche en Informatique de Toulouse (IRIT), University of Toulouse}\\
\normalsize{$^{2}$ Centre d'Economie de la Sorbonne (CES), University of Paris 1 Panthéon-Sorbonne}\\
\\
}

\date{ }

\begin{document}

\maketitle

\begin{abstract}
We introduce a voting model with multi-agent ranked delegations. This model generalises liquid democracy in two aspects: first,  an agent's delegation can use the votes of multiple other agents to determine their own---for instance, an agent's vote may correspond to the majority outcome of the votes of a trusted group of agents; second, agents can submit a ranking over multiple delegations, so that a backup delegation can be used when their preferred delegations are involved in cycles. The main focus of this paper is the study of unravelling procedures that transform the delegation ballots received from the agents into a profile of direct votes, from which a winning alternative can then be determined by using a standard voting rule. 
We propose and study six such unravelling procedures, two based on optimisation and four using a greedy approach. 
We study both algorithmic and axiomatic properties, as well as related computational complexity problems of our unravelling procedures for different restrictions on the types of ballots that the agents can submit. 
\end{abstract}

\section{Introduction}

In delegative voting, an agent's vote (or their corresponding voting power) can be passed to another voter or candidate. \citet{dodgson1885principles} was the first to mention a delegative voting process, in the context of multi-winner elections, where a candidate could strategically delegate their excess votes to another candidate of their choosing.

In general, models of delegative democracy bridge the gap between direct and representative democracy, where in the former every member of a community has to vote on every issue that arises, whereas the latter allows elected representatives to decide on behalf of a community. 
On the one hand, direct democracy is arguably time-consuming and infeasible for large-scale voting, since voters must be informed on every issue to be able to vote on it. On the other hand, representative democracy tends to leave voters under-represented (although it exists in many forms). 
 Delegative democracy can thus balance these problems, since agents can engage either \emph{actively}, by voting directly, or \emph{passively}, by choosing a representative for any issue  \cite{ford2002delegative}. 

Since voters can choose to actively participate into the decision-making, models of delegative democracy can also be seen as examples of \emph{interactive democracy}, i.e., those voting systems that turn collective decisions into more engaging and responsive processes. In particular,  \citet{brill2018interactive} argues that the progression of interactive democracy can be done in conjunction with the advancements of technology: for instance, in \emph{e-democracy} the Internet is used to strengthen real-world and online democracies \cite{shapiro2018point}. 
There are thus arguments in support of delegative democracy (and of upholding it via computerised methods), yet it is unclear how this should be implemented. 

\emph{Proxy voting} and \emph{liquid democracy} are two instances of delegative democracy. 
Proxy voting allows voters to choose their own representative who votes on their behalf. Depending on the model, the representatives can be predetermined \cite{alger2006voting,mueller1972representative};  any voter can be a representative and any agent can vote directly on any issue \cite{Green-Armytage2015,miller1969program, tullock1992computerizing,tullock1967toward}; or proxy voting is limited only to certain elections \cite{lanphier1995model}. 
%
%
Liquid democracy allows agents to either vote directly on an issue or to delegate their vote to another trusted agent. 
Unlike in proxy voting, delegations are transitive: i.e., if you have delegated your vote to another agent, they are free to either vote directly or to delegate their own vote (as well as all other delegations they have received) to another agent \cite{blum2016liquid}. 
However, 
the transitivity of delegations can lead to an agent's vote being used in a way that the agent would not support; for example, if a vote has been passed through many delegations, the original agent may not agree with the final agent who votes on their behalf. 
Another issue arises in determining the outcome when some agents are stuck in a delegation cycle, i.e., when a delegating agent transitively receives their own delegation, as well as all of those in the cycle.    

In this paper we tackle these problems by introducing a model of delegative democracy where voters can give multi-agent ranked delegations, thus generalising liquid democracy in two aspects. Firstly, we allow agents to submit a ranking over multiple possible delegations they would support, in order to ensure that their vote can be determined (should a delegation cycle arise). Secondly, voters can use the votes of multiple delegates in order to determine their own vote: for instance, they may state that their vote should coincide with the majority decision of a group of trusted delegates. 
Agents are thus able (but not required) to express complex delegations. 

A natural question which arises when allowing for ranked  delegations is when and how should the ranked delegations be used to break cycles. In this paper we define and analyse six of what we call \emph{unravelling procedures}, which are used to find outcomes from delegations that could contain cycles. 

\subsection{Our contribution}
In Section~\ref{sec:smartvotingmodel} we introduce our model of multi-agent ranked delegations, which we call \emph{smart voting}, building upon the model introduced in previous work
\cite{colley2020smart}. Our \emph{smart ballots} give more expressivity to the agents than those from previous models of delegative democracy. Starting from these complex ballots, an \emph{unravelling procedure} returns a standard voting profile from which a collective decision is taken. We introduce six unravelling procedures: $\minsum$ minimises the preference levels globally used for the agents; $\minmax$ minimises the maximum preference level used in the unravelling; the remaining four procedures use a greedy approach to guarantee tractability. We then introduce two restricted languages for smart ballots on binary issues: $\Blang$ is our general language, where delegations are arbitrary Boolean functions expressed in complete DNF, while $\Liqlang$ allows for ranked single-agent delegation.  
 
 In Section~\ref{sec:compucomplex} we prove our main results: the decision problems needed to compute $\minsum$ and $\minmax$ on the $\Blang$ language are \textsc{NP}-complete, though they are polynomial for $\Liqlang$ ballots. Moreover, we prove that our four greedy unravelling procedures always terminate on valid smart ballots in polynomial time. 

In Section~\ref{sec:comparingunravelling} we compare our unravelling procedures, showing that all six can give outcomes which differ from each other and from the procedures by \citet{kotsialou2018breadth}. However, the four greedy procedures and $\minsum$ coincide on standard liquid democracy ballots. We also study the axioms of cast-participation and guru-participation, as well as a notion of Pareto dominance. We conclude in Section~\ref{sec:conc}.

\subsection{Related work}

In this section we present the related literature on delegative democracy, starting from the previous work that inspired our twofold generalisation of liquid democracy.
 
\paragraph{Multi-agent delegations.} Our first generalisation is to allow voters to express delegations involving many other agents.
In the same spirit, the model by \citet{degrave2014resolving} allows for an agent's delegation to be split among possibly many delegates: for example, a voter delegating to three agents could give half of their vote to one agent and a quarter to the other two delegates. \citet{AbramowitzM19} give a similar model, but for proxy voting, 
 where agents assign weights to a fixed set of representatives for each issue: agents can thus choose to spread their vote over many representatives.
 
  In pairwise liquid democracy on ordinal elections \cite{BrillTalmonIJCAI2018}, ballots are rankings of candidates and they can be completed by delegating a decision on distinct pairs of candidates to different delegates. 
\citet{golz2018fluidliquid} address the problem of a small number of agents holding a lot of power by studying the fluid mechanics of liquid democracy, trying to balance influence like liquid in a vessel.

\paragraph{Ranked delegations.}
The second generalisation is to allow for ranked delegations, to avoid cycles (as in our case) or to avoid the delegation to an abstaining agent.  
\citet{ford2002delegative} introduces the notion of \emph{delegation chains}, which we call \emph{ranked delegations}, where an agent submits an ordering of many delegates, from the most trusted to one who still represents the agent, but less than the delegates coming before. 


The model proposed by \citet{kotsialou2018breadth} includes ranked delegations that are used to avoid delegating to an abstaining agent and delegation cycles. They propose two procedures, breadth-first and depth-first, which find an outcome in a similar manner as we suggest; yet they do not allow delegations to abstaining agents. 
A similar approach has been proposed by \citet{BrillPresentation} in ongoing work.

 \citet{kahng2018liquid} suggest an alternative method of removing delegation cycles, by assigning a competency level to agents and forbidding them to delegate to someone with a lower competency level than themselves. \citet{boldi2011viscous} propose \emph{viscous democracy}, a model of liquid democracy that can be extended to include ranked delegations, where they make use of a dampening factor to ensure short chains of delegations. 
\citet{behrens2015finite} propose a model of ranked liquid democracy as well as seven desirable properties for a liquid democracy system, showing that no system can satisfy all of them.

As an alternative to extending liquid democracy with ranked delegations, \citet{EscoffierEtAlSAGT2019} introduce iterative delegations, to circumvent situations where outcomes cannot be determined. For example, an agent may want to change their vote if they are in a delegation cycle or if they do not approve of the \emph{guru} (i.e., the final receiver of the delegation) that votes on their behalf. However, stable states may not exist: \citet{escoffier2020iterative} study how the profile of preferences impacts the existence of a stable state, showing that some structures on profiles (such as single-peaked preferences) can ensure an equilibrium.



\paragraph{Other models of delegative democracy.}
Inspired by \citet{kahng2018liquid},
\citet{caragiannis2019contribution} study the interplay of liquid democracy and truth-tracking. Assuming that voting is about finding some ground-truth, a delegation to another agent implies that you believe they will be better than you at finding the truth. Finding optimal delegations (obtaining the highest probability of finding the ground-truth) is however not a tractable problem and it is also hard to approximate.

Analogously, \citet{cohensiusproxy} use proxy voting to approximate an underlying ground-truth, when voter participation would otherwise be low, and show via empirical data that the outcomes are more accurate than with full voter participation. 



\citet{christoff2017binary} give a model of liquid democracy on multiple binary issues connected by constraints that reflect consistent sets of opinions. In this extension of liquid democracy an agent's vote acquired via delegation must be consistent with the constraint and their votes on other issues. They also include an embedding of transitive delegative voting on binary issues into binary aggregation with abstentions. 
\citet{BGL19} give a game-theoretic analysis of delegation games via pure-Nash equilibria, where the agent's utility is the accuracy of their vote.  They complement this study with agent-based simulations that verify their theoretical best-response dynamic. 
\citet{CohensiusEtAlAAMAS2017} consider a game-theoretic model, where only a small number of agents are able or motivated to vote. They compare voting with and without proxies, studying which method approximates the optimal outcome. 
Their results support proxy voting when agents are placed on a line; however, in the general setting, 
some agents may gain too much power. 

\citet{zhang2020power} also study delegation games and provide a generalization of the Banzhaf index in liquid democracy to investigate the power of the agents who vote directly on the issues. 
\citet{meir2020sybil} study power in proxy voting, trying to limit the power of sybil-voters in their system. 
They find bounds on the highest number of sybil-voters a system could have while still upholding two forms of sybil-resistance. 


\paragraph{Social choice on social networks.} Delegative democracy can be seen as a form of social choice on social networks  \cite{grandi2017social}, since delegations create a social network where an edge represents the trust one agent has for another.  This is also seen in models of opinion diffusion: in both cases, trusted agents can affect either the opinion or the vote of another agent. Threshold models of opinion diffusion \cite{granovetter1978threshold} are closely related to multi-agent delegations using quota rules; one such model is that of \citet{DBLP:conf/ijcai/BredereckE17}, where opinions change to be in agreement with the majority of the agent's neighbours. 
 The pairwise diffusion model of \citet{DBLP:conf/ijcai/BrillEEG16} is closely related to the one of pairwise liquid democracy \cite{BrillTalmonIJCAI2018}: agents give an ordering over all alternatives for one issue, and the structure of both the network and the agent's preferences are central to determine the termination of the diffusion process.

\paragraph{Advancements in voting technology.}


A challenge of delegative democracy is the technology required to implement it. \citet{miller1969program} in \citeyear{miller1969program} proposed that voters would need  ``\emph{a special metal key, a coded combination, or even a thumbprint}'' to ensure a safe voting process. Since then, the advancement of technology has surpassed these notions, yet the question remains of how to safeguard voting when relying on technology. 

A suggestion is to use blockchain technology, such as smart contracts. 
 \citet{dhillon2019introduction} give a detailed plan of the infrastructure required for a decentralised online voting platform, using distributed ledgers---showing the strengths and weaknesses of such a system. 
\citet{zhang2019statement} give a model of \emph{statement voting} that uses the Universal Composability framework to circumvent the issues of implementing $e$-democracies. Their model allows for approval voting, STV, and liquid democracy. 

Along with theoretical advancements, voting platforms have also become more commonplace---including those designed for liquid democracy. \citet{mancini2015time} advocates for political systems to be brought into line with the technology available, and that liquid democracy could ``\emph{rewire from the inside how politics works}''. Liquid democracy has also been promoted by political parties, such as the \emph{Partido de Internet} in Mexico, the \emph{Net Party} in Argentina, and most notably the \emph{Pirate party of Germany} who implemented the Liquid Feedback software \cite{BehrensEtAl2014} for internal deliberation. Voting behaviour has been studied on the platform \cite{kling2015voting}, and the party's structure has been altered by its use \cite{litvinenko2012social}.
Liquid Feedback introduced \emph{pseudonymity} to have some level of anonymity in the system via pseudonyms \cite{swierczek2014five}. 

Finally, \emph{Adhocracy} has a voting platform that uses proxy voting for communities, whereas \emph{Ethereum} has a voting platform for liquid democracy. Moreover, Google tested liquid democracy in their internal social network via \emph{Google Votes}~\cite{HardtLopes}.
\section{Smart voting}\label{sec:smartvotingmodel}

In this section we recall the definitions of our model for multi-agent ranked delegations in voting, which has been previously introduced as \emph{smart voting} \cite{colley2020smart}. The model allows agents to decide if they want to vote directly on the issues at stake, or to give (possibly complex and/or multiple) delegations to determine their vote from the votes of others in the electorate.

In a smart voting election, the following four main stages will take place:

\begin{enumerate}
\item Each agent participating in the election creates and sends their own smart ballot (as described in Section~\ref{subsec:smart ballots}), which could be restricted to a specific format; 
\item Once the central authority has received all the ballots, they check whether the ballots abide by the aforementioned restrictions (i.e., whether they are \emph{valid} ballots);
\item Since smart ballots may include delegations, they need to be unravelled via some procedure to obtain a standard voting profile (as presented in Section~\ref{subsec:unravel});
\item Finally, a classical voting rule, such as the majority or plurality rule, is used on the resulting standard voting profile to obtain the collective decision.
\end{enumerate}

In the next sections, we will study problems which pertain to the steps 1--3 of a smart voting election. For the final step 4, any standard voting rule can be used.

\subsection{Smart ballots}\label{subsec:smart ballots}

In a smart voting election, a finite set $\N$ of $n$ \emph{agents} (or \emph{voters}) has to take a collective decision over a finite set $\I$ of $m$ independent issues. The possible values, or \emph{alternatives}, for each issue $i \in \I$ range over a non-empty finite \emph{domain} $\Dom(i)$, which can also include abstentions (denoted by the symbol $*$). 

Each agent $a \in \N$ expresses her vote over an issue $i\in\I$ by submitting what we call a \emph{smart ballot} $B_{ai}$, defined as follows:

 \begin{definition}[Smart ballot]\label{def:smartballot}
   	 A \emph{smart ballot} of agent $a$ on an issue $i \in \I$ is an ordering $((S^1,F^1)>\dots>(S^k,F^k)>x)$ where $k\geq 0$ and for $ h \le k$ we have that $S^h \subseteq \N$ is a set of agents, $F^h : \Dom(i)^{S^h}\rightarrow \Dom(i)$ is a resolute aggregation function and $x \in \Dom(i)$ is an alternative.
\end{definition}

For simplicity, we will focus on a single issue $i\in\I$. We thus drop the index $i$ throughout to simplify notation---writing, e.g., $\Dom$ instead of $\Dom(i)$ and $B_a$ instead of $B_{ai}$. All our results can be easily generalised to multiple independent issues.

In a smart ballot, an agent is expressing a preference ordering over their desired delegations---i.e., the $k$ domain-function pairs $(S,F)$---which ends with a direct backup vote $x$ for an alternative in the domain of the issue. Note that while it is required for an agent to provide a backup vote as their last choice, that vote could be an abstention (if it is in the domain of the issue). Moreover, an agent can submit a smart ballot with $k=0$, i.e., they vote directly on the issue without any delegation.

In most of our examples, a delegating function $F$ either takes the form of 
 a single-agent delegation,  an aggregation rule (such as the majority rule), or a Boolean function.
 In the latter case we assume that the function is represented as a Boolean formula built from a set of propositional variables $\{a\mid a\in \N\}$, representing the use of agent $a$'s vote to determine the delegation, and the standard logical connectives for negation $\neg$, conjunction $\wedge$, and disjunction $\vee$, on a binary domain.

 When computing a function $F$ on an incomplete input we use the notion of a \emph{necessary winner} \cite{konczak2005voting}. For example, consider an agent $a\in\N$ having a delegation $\mathit{Maj}(b,c,d)$---i.e., the majority over the votes of $b$, $c$ and $d$---on a binary issue for which $b$ and $c$ have a direct vote for $1$, while there is no vote of $d$ yet. We can still compute $\mathit{Maj}(b,c,d)$ without the vote of $d$, since there is already a majority for $1$.

From our general definition of smart ballots, we further distinguish a class of \emph{valid} smart ballots satisfying some desirable properties:

   \begin{definition}[Valid smart ballot]\label{def:validsmartballot}
   	 A \emph{valid smart ballot} of agent $a$ is a smart ballot $B_{a}$ such that, for all $1 \le s \neq t \le k$, we have that $(i)$ if $S^s\cap S^t \neq \emptyset$ then $F^s$ is not equivalent to $F^t$, and $(ii)$ $a \notin S^s$.	
   \end{definition}

Condition $(i)$ imposes that an agent is not submitting the same delegation at multiple preference levels of their smart ballot. 
This removes the possibility of manipulating the process by delegating to the same agents with the same delegation function multiple times, thus ensuring that the more preferred  delegation or an equivalent formulation will be chosen. 
%
Condition $(ii)$ ensures that agents cannot include themselves in the set of delegates, as this would immediately lead to a delegation cycle.

The following example illustrates various instances of possible smart ballots:

\begin{example}\label{ex:smartballot}

Six agents $\N=\{a,b,c,d,e,f\}$ face the decision of whether to order dinner from one of two new restaurants. Let us denote by $1$ the first restaurant, by $0$ the second one, and assume that agents can also abstain, i.e., $\Dom = \{1,0,*\}$. 

We now present some options of valid smart ballots that agent $a$ could submit:

  \begin{itemize}
    \item[$(i)$] $B_a=(1)$

    This smart ballot represents the direct vote of $a$ for the first restaurant.
    \item[$(ii)$] $B_a=((\{b,c,d,e,f\}, \mathit{RMaj})>0)$

Here, agent $a$ wants their vote to be the \emph{relative majority} $\mathit{RMaj}$ (i.e., plurality between $0$ and $1$, with $*$ in case of a tie) of the choices of the agents $b,c,d,e,f$. If this first delegation of $a$ leads to a cycle, they will vote for the second restaurant.
 
    \item[$(iii)$] $B_a=((\{d\},id)> (\{e\}, id)> *)$

Here, $a$'s first preference is to delegate to agent $d$ ($id$ indicates the identity function); if this causes a delegation cycle, then $a$ chooses to delegate to $e$; and if this also causes a cycle, then $a$ abstains from the vote. 
    \end{itemize}
    
    Suppose now that we have a binary issue\footnote{Note that in order to use Boolean functions to express a delegation, the domain of the alternatives for the issue must be a Boolean algebra. We restrict this to a two-element Boolean algebra, namely $\{0,1\}$.} with domain $\Dom=\{0,1\}$ where~$1$ represents ordering take-away while $0$ represents cooking at home. 
    \begin{itemize}

    \item[$(iv)$]   $B_a=((\{b, f\}, b\vee f )>(\{b,c,e\}, (c\wedge b) \vee (\neg e\wedge b))>1)$

In this case, $a$'s first choice is to delegate using the Boolean function $b\vee f$, i.e., $a$ will vote for take-away if either $b$ or $f$ want to; if this creates a delegation cycle, then $a$ will vote to try the take-away if both $b$ and $c$ also want to, or $b$ wants to while $e$ prefers to cook at home. If this still causes a cycle, then $a$ votes for $1$.
    \item [$(v)$] $B_a=((\{b,c, f\}, \mathit{Maj} )>(\{b,c,e\}, (c\wedge b) \vee (\neg e\wedge b))>1)$

This smart ballot differs from $(iv)$ only in that $a$'s first preference is to have their vote coincide with the majority of the agents $b,c$ and $f$'s votes.
  \end{itemize}
  One can easily check that all the ballots in Example~\ref{ex:smartballot} are valid as per Definition~\ref{def:validsmartballot}.
  
\end{example}

Each linear order of delegations (plus the backup vote) in a smart ballot indicates a preference over possible delegations. We write $B_a^h$ to indicate the $h^\text{th}$ \emph{preference level} given by agent $a$ in their smart ballot $B_a$. Hence, we have $B_a^h =(S_a^h,F_a^h)$ when the $h^\text{th}$ preference level is a delegation, or $B_a^h=x$ with $x \in \Dom$ when the $h^\text{th}$ preference level is a direct vote. In Example~\ref{ex:smartballot}, e.g., 
$B_a^2=(\{e\},id)$ for ballot $(iii)$.

We collect the $n$ smart ballots from each agent in $\N$ into a \emph{(smart) profile}, i.e., a vector $\boldsymbol{B} = (B_1, \dots, B_n)$. A \emph{valid (smart) profile} is a smart profile where each smart ballot is valid, according to Definition~\ref{def:validsmartballot}. 

\subsection{Unravelling procedures}\label{subsec:unravel}

An \emph{unravelling procedure} is a function which allows us to turn a smart profile into a standard voting profile, i.e., a vector in $\Dom^n$ of direct votes for the issue.

\begin{definition}[Unravelling procedure]
  An \emph{unravelling procedure} $\mathcal{U}$ for the agents in~$\N$ is any computable function 
  $$\mathcal{U}: (B_1 \times \dots \times B_n) \rightarrow \Dom^n.$$
\end{definition}

Thus, an unravelling procedure $\mathcal{U}$ takes a smart profile $\Prof$ and it returns a voting profile in $\Dom^n$. When $\mathcal{U}$ and $\Prof$ are clear from context, we often write just $X$ to denote an outcome of an unravelling procedure: i.e., $X \in \mathcal{U}(\Prof)$ for $X \in \Dom^n$.

After an unravelling procedure returns a profile of direct votes, the agents may want to know how their smart ballot was unravelled: i.e., which preference level of their ballot was actually used to compute their direct vote. For this purpose, we introduce the notion of a \emph{certificate}.

\begin{definition}[Certificate]\label{def:certificate} A certificate $\cert\in\mathbb{N}^n$ for profile $\Prof$ is a vector where, for all $a \in \N$ such that $B_a = (B_a^1 > \dots > B_a^{k_a})$, the entry $\cert_a\in [1, k_a]$ corresponds to a preference level for agent $a$.

\end{definition}

Within the class of all possible certificates of Definition~\ref{def:certificate}, we are interested in those that satisfy the following property: a certificate is \emph{consistent} if there is an ordering of the agents such that an agent's vote can be determined using the preference level in the certificate, given the votes of the agents that come prior in the order.\footnote{This notion, when restricted to ballots with single-agent delegations, corresponds to the definition of confluent sequence rules by \citet{BrillPresentation}.} 

\begin{definition}[Consistent certificate]\label{def:ConsistentCert} Given a profile $\Prof$ of valid ballots, a certificate $\cert$ is \emph{consistent} if there exists an ordering $\boldsymbol{\sigma} : \N \to \N$ of the agents which, starting from vector $X^0=\{\Delta\}^n$ with placeholder values $\Delta$ for all agents, iteratively constructs an outcome vector of direct votes $X\in\Dom^n$ as follows, for $\sigma(a) = z \in [1,n]$:

$$
X^z_a =
\begin{cases}
B_a^{\cert_a} & \mbox{ if } B_a^{\cert_a} \in \Dom \\
F^{\cert_a}_a(X^{z-1}\rest_{S_a^{\cert_a}}) & \mbox{ otherwise }  
\end{cases}
$$
 where $X_a$ represents $a$'s entry in $X$ and $X \rest_S = \Pi_{s\in S}X_s$.



\end{definition}

We let $\C{\Prof}$ be the set of all consistent certificates of a profile $\Prof$, and the $a^\text{th}$ entry of $\cert$ corresponds to $B_a^{\cert_a}$ being used by the unravelling procedure.
Moreover, when outcomes of unravelling procedures can be determined by certificates we will use  $\Crest{\mathcal{U}}{\Prof}$ to denote all consistent certificates given by the unravelling procedure $\mathcal{U}$. 
We show next that each consistent certificate $\cert$ has a unique corresponding outcome vector $X_\cert$ of direct votes.

\begin{proposition}
If a consistent certificate $\cert$ can be given by two orderings $\sigma$ and $\sigma'$ of the agents (as per Definition~\ref{def:ConsistentCert}), then the orderings yield the same outcome $X_\cert\in\Dom^n$.  
\end{proposition}
\begin{proof}
Consider an arbitrary profile $\Prof$ and a consistent certificate $\cert \in \Crest{}{\Prof}$. Assume for a contradiction that $\cert$ can yield two distinct vectors of direct votes $X\neq X'$, which are given by two orderings $\sigma$ and $\sigma'$ of $\N$, respectively. 
To reach a contradiction, we show by induction on the ordering $\sigma$ that for each agent $a\in\N$ we have $X_a = X'_a$.

 For the base case, consider agent $a\in\N$ such that $\sigma(a) = 1$.  
 As $a$'s vote was added to $X_\cert$ without any other vote, $\cert_a$ must refer to a direct vote.  Therefore, the direct vote of $a$ will be added to $X$ and $X'$ (although it may be that $\sigma'(a) \neq 1$). 
We assume for our inductive hypothesis that for all agents $b\in \N$ where $\sigma(b)\leq k$ it is the case that $X_b=X'_b$. We will show that for agent $d$ such that $\sigma(d)=k+1$ we have $X_d=X'_d$. 
In case $B_d^{\cert_d}$ is a direct vote, the same reasoning as for the base case applies. Else, by definition we have a necessary winner for $F_d^{\cert_d}(X\rest_{S_d^{\cert_d}})=X_d$. If $X_d\neq X'_d$, then $F_d^{\cert_d}(X\rest_{S_d^{\cert_d}})\neq F_d^{\cert_d}(X'\rest_{S_d^{\cert_d}})$ and $X\rest_{S_d^{\cert_d}}\neq X'\rest_{S_d^{\cert_d}}$. Hence, there exists an entry that differs in the two vectors, which contradicts our inductive hypothesis. Then, $X_d=X'_d$.
As $X_a=X'_a$ for all $a\in \N$, we have that $X=X'$. Hence, a consistent certificate $\cert$ gives a unique outcome $X_\cert$. 
\end{proof}

Finally, we define the \emph{rank} of a certificate $\cert$ as the sum of its preference levels used.  Given profile $\Prof$, the $\rank$ of a certificate $\cert\in \mathcal{C}(\Prof)$ is $\rank(\cert):= \sum\limits_{a \in \N}\cert_a$.

The minimum possible value of $\rank$ for an unravelling is $n$, i.e., when all the agents get their first preference level. Thus, if a profile contains a delegation cycle at the first preference level, it cannot have a consistent certificate with rank equal to $n$.

\subsection{Optimal unravellings}\label{subsec:optunravel}
Our first procedure, $\minsum$, is \emph{optimal} with respect to the $\rank$: i.e., it returns all outcome vectors which can be obtained by a consistent certificate minimising the sum of preference levels used for the agents.

\begin{definition}[MinSum]\label{def:opt}
 For a given profile $\Prof$, the $\minsum$ unravelling procedure is defined as:
$$\minsum(\Prof):= \{ X_\cert \mid \cert \in \argmin\limits_{\cert\in \mathcal{C}(\Prof)}\rank(\cert)\}.$$
\end{definition}

Hence, $\minsum$ returns all vectors of direct votes $X_\cert$ whose consistent certificate~$\cert$ minimises the value of $\rank(\cert)$. Intuitively, by minimising the agents' preference levels used, more trusted agents are being delegated to. 
Next, we give examples of consistent certificates and of the outcomes of the $\minsum$ procedure.

\begin{example}\label{ex:rank}
Consider a binary issue with domain $\Dom = \{0,1\}$ and five agents $\N = \{a, b,c,d, e\}$, whose ballots form the profile $\Prof$, shown schematically in Table~\ref{table:OptExPreference}.

\begin{table}[h]
\centering
 \begin{tabular}{llll} 
 \hline\noalign{\smallskip}
& $B_x^1$ &$B_x^2$&$B_x^3$\\
 \noalign{\smallskip}\hline\noalign{\smallskip} 
$a$&$(\{b,c\}, b \land c)$&$(\{d\},d)$&$1$\\
$b$&$1$&-&-\\
$c$&$(\{d\}, d)$&0&-\\
$d$&$(\{e\},e)$&$1$&-\\
$e$&$(\{a\},a)$&$(\{b\}, b)$&0\\
\noalign{\smallskip}\hline
 \end{tabular}
  \caption{The smart profile $\Prof$ of Example~\ref{ex:rank}, where each row represents the ballot for each of the agents $\N=\{a, b,c,d, e\}$, while the columns separate the different preference levels of the agents' ballots.}
\label{table:OptExPreference}

\end{table}

Note that there is a delegation cycle at the first preference level $B_x^1$: agent $a$ needs the vote of $c$ to compute their own vote, agent $c$ delegates to $d$, agent $d$ delegates to $e$, and agent $e$ delegates to $a$.
 Hence, the certificate vector $\cert=(1,1,1,1,1)$, which would be the minimal one for this profile, is not consistent: there is no ordering of the agents where their direct votes are computed using only their first preference levels. Thus, $\cert\notin \C{\Prof}$ and the value of $\rank$ for a consistent certificate will be at least $6$.

Consider the certificate $\cert'=(1,1,2,1,1)$, where only $c$ has their second preference used: $\cert'$ is consistent, as it is shown by the ordering $\sigma = (b, c, a, e, d)$. As $\rank(1, 1, 2, 1, 1) = 6$, the corresponding outcome $X_{\cert'} = (0,1,0,0,0)$ is an outcome of $\minsum(\Prof)$. The consistent certificate $\cert'' = (1,1,1,2,1)$ gives $X_{\cert''} = (1,1,1,1,1)$ and as $\rank(\cert'') = 6$, it also is in $\minsum(\Prof)$. Since there can be multiple certificates minimising the total rank (yielding distinct vectors of direct votes~$X$) the $\minsum$ unravelling procedure is not resolute.
\end{example}

While $\minsum$ maximises the \emph{global} satisfaction of the agents, from an \emph{individual} perspective there can be a large disparity in the selected preference levels.
Our second optimal procedure is motivated by an \emph{egalitarian} approach, finding outcomes whose certificate minimises the maximum preference level used among the agents.

\begin{definition}[MinMax]\label{def:minmax}
Given profile $\Prof$, the $\minmax$ unravelling procedure returns the following vectors of direct votes:

$$\minmax(\Prof):= \{ X_\cert \mid \cert \in \argmin\limits_{\cert\in \C{\Prof}}\max(\cert)\}.$$
\end{definition}


\begin{example}\label{ex:MinMaxbenefit}
Consider a binary issue and $26$ agents $\N=\{a,\dots, z\}$. Let the profile $\Prof$ be such that the smart ballot of agent $a$ is $B_a=((\N\backslash\{a\}, \bigvee_{x\in\N\backslash\{a\}} x)>(\N\backslash\{a,b\}, \bigvee_{x\in\N\backslash\{a,b\}} x)>(\N\backslash\{a,b,c\},\bigvee_{x\in\N\backslash\{a,b,c\}} x)>1)$, and for each agent $x\in \N\setminus\{a\}$ let $B_{x}=((\{a\},a)>0)$ be their smart ballot.

There are three outcomes of $\minmax(\Prof)$, with certificates $\cert=(1,2,\dots, 2)$, $\cert'=(2,\dots,2)$, $\cert'' = (2,1,2,\dots,2)$, where $\max(\cert)=\max(\cert')=\max(\cert'')=2$, even though $\rank(\cert)=\rank(\cert'')=51$ and $\rank(\cert')=52$.
The outcome of $\minsum(\Prof)$ has certificate $\cert'''=(4,1,\dots,1)$ and $\rank(\cert''')=29$; however, this is not an outcome of $\minmax$, since $\max(\cert''')=4$.
\end{example}

A disadvantage of $\minmax$ is that for some profiles $\Prof$ it may return a large number of tied outcomes, as we shall see in Example~\ref{prop:unravelldifferent}. 


\subsection{Greedy unravellings}\label{subsec:greedyunrav}
In Section~\ref{sec:compucomplex} we will prove that computing an outcome for $\minsum$ and $\minmax$ is not computationally tractable in general. This motivates us to introduce four unravelling procedures with a greedy approach, that break delegation cycles by using the lowest possible preference level of the ballots, while keeping the process tractable.

\begin{algorithm}[h]
\caption{General unravelling procedure  \textsc{Unravel}}\label{Alg:Gen}
\begin{algorithmic}[1]
 \State Input: $\Prof$  
  \State $ X:=(\Delta, \dots, \Delta)$ \Comment{\emph{vector for direct votes initialised with placeholders $\Delta$}} \label{line:initvector}
  \While{$X\notin \Dom^n$ } \label{1st:line:repeat1}
    \State $\lev:=1$  \Comment{\emph{reset preference level counter $\lev$ to 1}}\label{line:initlev}
    \State $Y:=X$ \Comment{\emph{store a copy of $X$ to compute changes}} \label{line:countery} 
      \While{$X= Y$}  \label{2nd:line:repeat2}
      \Procedure{Update}{$\#$}  with $\#\in \{\mathbf{U},\mathbf{RU},\mathbf{DU},\mathbf{DRU}\}$       \label{line:update} 
          \EndProcedure
          \State $\lev:=\lev+1$ \label{line:nextpref}
      \EndWhile \label{3rd:line:until2}
 \EndWhile\label{4th:line:until1} 
 \State \Return{$X$}\Comment{\emph{output a vector of direct votes}}
\end{algorithmic}
\end{algorithm}

Algorithm~\ref{Alg:Gen} outlines our \emph{general unravelling procedure} \textsc{Unravel}. The input is a smart profile $\Prof$, and the procedure initialises a vector $X$ with placeholders $\Delta$ for each agent $a \in \N$. The outcome $X$ is returned when each agent has a vote in $\Dom$, i.e., $X\in \Dom^n$. A counter \lev\, is always reset to $1$ to come back to the first preference level of the agents. An additional vector $Y$ is used to help with intermediate computations.

In line~\ref{line:update} a subroutine using an \emph{update procedure} is executed.\footnote{In the following, we will simply write \textsc{Unravel}($\#$), for $\# \in \{{\bf U}, {\bf DU}, {\bf RU}, {\bf DRU} \}$, to indicate the \textsc{Unravel} algorithm using \textsc{Update} procedure $\#$.} Given a partial profile of direct votes and placeholders $\Delta$, as well as a preference level \lev, the \textsc{Update} procedure searches for a direct vote or a vote that can be computed via necessary winners (depending on which \textsc{Update} is used) at the \lev$^\text{th}$ preference level in the profile; if this is not possible, \textsc{Unravel} moves to level $\lev + 1$ (line~\ref{line:nextpref}). 

The four update procedures that could be called in Algorithm~\ref{Alg:Gen} are defined by the presence or absence of two properties. The first is \emph{direct vote priority} (D): an update procedure prioritises \emph{direct votes} over those that can be computed from the current vector $Y$ of votes. The second is \emph{random voter selection} (R): an update procedure \emph{randomly} selects, when possible, a single agent whose direct or computable vote can be added to $X$. We thus get: basic update ($\mathbf{U}$), update with direct vote priority ($\mathbf{DU}$), update with random voter selection ($\mathbf{RU}$), update with both direct vote priority and random voter selection ($\mathbf{DRU}$).

\begin{algorithm}[h]
\caption{\textsc{Update}($\mathbf{U}$)}\label{Alg:U}
\begin{algorithmic}[1]
      \For{$a \in \N$ such that $x_a = \Delta$} \label{line:loopwodirect}
        \If{$B_a^{\lev}\in \Dom$} \Comment{\emph{add $a$'s vote if $a$ has a direct vote at $\lev$}}
        \State $x_a:=B_a^{\lev}$   \label{lineU:votedirect}
        \ElsIf{$F_a^{\lev}(Y_{\restriction S_a^\lev})\in \Dom$}
        \State$x_a:=F_a^{\lev}(Y_{\restriction S_a^\lev} )$  \Comment{\emph{add $a$'s vote if $a$ has a computable vote at $\lev$}} \label{lineU:votecomputable}
        \EndIf
      \EndFor
\end{algorithmic}
\end{algorithm}
The \textsc{Update}($\mathbf{U}$) procedure\footnote{Unless otherwise specified, in case the condition in an {\bf if} statement fails, our programs will skip to the next step. Recall also that $Y_{\restriction S}$ denotes the restriction of vector $Y$ to the elements in set $S$.} in Algorithm~\ref{Alg:U} updates the vector $X$ with the direct votes for those agents who currently do not have one (line~\ref{line:loopwodirect}), if their preference at $\lev$ is a direct vote (line~\ref{lineU:votedirect}) or it can be computed from the current votes in $Y$ (line~\ref{lineU:votecomputable}).

\begin{algorithm}[h]
  \caption{\textsc{Update}($\mathbf{DU}$)}\label{Alg:DU}
\begin{algorithmic}[1]
      \For{$a \in \N$ such that $x_a = \Delta$}
        \If{$B_a^{\lev}\in \Dom$} \Comment{\emph{add all direct votes}}\label{line:DUdirect}
        \State $x_a:=B_a^{\lev}$   
        \EndIf
      \EndFor
      \If{$Y= X$}\Comment{\emph{if no direct votes are added to $X$}}  \label{lineDU:comp}
      \For{$a \in \N$ such that $x_a = \Delta$}
        \If{$F_a^{\lev}(Y_{\restriction S_a^\lev})\in \Dom$} \Comment{\emph{find and add computable votes to $X$}}  \label{lineDU:compadd}
        \State $x_a:=F_a^{\lev}(Y_{\restriction S_a^\lev} )$  
        \EndIf
      \EndFor
      \EndIf
\end{algorithmic}
\end{algorithm}

In Algorithm~\ref{Alg:DU}, \textsc{Update}($\mathbf{DU}$) first adds the direct votes from preference level $\lev$ to $X$ for those agents without a vote in $X$ (line~\ref{line:DUdirect}). If there are no direct voters at $\lev$  (line~\ref{lineDU:comp}), then the procedure tries to add computable votes (line~\ref{lineDU:compadd}).


\begin{algorithm}[h]
  \caption{\textsc{Update}($\mathbf{RU}$)}\label{Alg:RU}
\begin{algorithmic}[1]
  \State $P:= \emptyset$\Comment{\emph{initialise an empty set}} \label{line:initprobset}
    \For{$a \in \N$ such that $x_a = \Delta$} \label{line:forstartP}
      \If{$B_a^{\lev}\in \Dom$ or $F_a^{\lev}(Y_{\restriction S_a^\lev})\in \Dom$ } \label{line:addtoP}
      \State  $P:=P\cup\{a\}$\Comment{\emph{add voters to $P$ if their vote can be determined}}
      \EndIf
    \EndFor\label{line:forendP}{}
    \If{$P\neq \emptyset$}\Comment{\emph{there are direct or computable votes in $P$}} \label{line:Pnotempty}
    \State {\bf select} $b$ from $P$ uniformly at random \label{line:lottery}
      \If{$B_b^{\lev}\in \Dom$} 
      \State $x_{b}:=B_b^{\lev}$ \label{lineRU:directaddY}
      \ElsIf{$F_b^{\lev}(Y_{\restriction S_b^\lev})\in \Dom$} 
      \State $x_{b}:=F_b^{\lev}(Y_{\restriction S_b^\lev})$ \label{lineRU:compaddY}
    \EndIf 
    \EndIf
\end{algorithmic}
\end{algorithm}
The \textsc{Update}($\mathbf{RU}$) procedure has the random voter selection property (Algorithm~\ref{Alg:RU}): at line~\ref{line:initprobset} an empty set $P$ is initialised to store agents with either a direct vote or a computable vote at $\lev$ (line~\ref{line:addtoP}). If $P$ is non-empty, one agent will be randomly selected and their direct or computable vote will be added to $X$. 

\begin{algorithm}
  \caption{\textsc{Update}($\mathbf{DRU}$)}\label{Alg:DRU}
\begin{algorithmic}[1]
    \State $P,P':= \emptyset$ \Comment{\emph{initialise two empty sets}}
      \For{$a \in \N$ such that $x_a = \Delta$}  
        \If{$B_a^{\lev}\in \Dom$} \Comment{\emph{add agents with direct votes at $\lev$ to $P$}} \label{lineDRU:directvote}
        \State $P:=P\cup \{a\}$
        \ElsIf{$F_a^{\lev}(Y\restriction_{S_a^\lev})\in \Dom$}\Comment{\emph{add agents with computable votes at $\lev$ to $P'$}}
        \State $P':=P'\cup \{a\}$
        \EndIf
      \EndFor
      \If{$P\neq \emptyset$} \Comment{\emph{if there are agents with direct votes}}
      \State {\bf select} $b$ from $P$ uniformly at random 
      \State $x_{b}:=B_b^{\lev}$   \Comment{\emph{add only the randomly selected voter's direct vote to $X$}} \label{lineDRU:randomdirect}
      \ElsIf{$P'\neq \emptyset$} \Comment{\emph{if there are some computable votes}}
        \State {\bf select} $b$ from $P'$ uniformly at random 
        \State  $x_{b}:=F_b^{\lev}(Y\restriction_{S_b^\lev} )$ \Comment{\emph{add only the randomly selected voter's computable vote to $X$}}\label{lineDRU:randomcomputable}
      \EndIf
\end{algorithmic}
\end{algorithm}

Lastly, Algorithm~\ref{Alg:DRU} presents \textsc{Update}($\mathbf{DRU}$), which has both properties. At $\lev$, the procedure adds agents with direct votes to $P$ (line~\ref{lineDRU:directvote}) and agents with computable votes to $P'$. If $P$ is not empty, an agent is selected from $P$ and their direct vote is added to $X$ (line~\ref{lineDRU:randomdirect}). Otherwise, if $P$ is empty and $P'$ is not, an agent is selected from $P'$ and their computable vote is added to $X$ (line~\ref{lineDRU:randomcomputable}). If both $P$ and $P'$ are empty, no votes are added to $X$ and the procedure terminates. 

We now give an example of the execution of these four unravelling procedures.

\begin{example}\label{ex:unravellingprocedures}

For a binary issue with $\Dom = \{0,1\}$ consider agents $\N = \{a, \dots, f\}$, whose ballots and delegation structure are represented schematically in Figure~\ref{fig:UnravellingExample}.\footnote{Observe that a formula of propositional logic is a Boolean function.} First of all, $\Prof$ is thus a valid profile. We now illustrate our four unravelling procedures for $ \textsc{Unravel}(\#)$ with $\#\in\{\mathbf{U},\mathbf{DU},\mathbf{RU},\mathbf{DRU}\}$.

\begin{figure}[h]
\begin{subfigure}{0.58\textwidth}
\begin{tabular}{llll} 
\hline\noalign{\smallskip} 
&$B_x^1$&$B_x^2$&$B_x^3$\\
 \noalign{\smallskip}\hline\noalign{\smallskip}
$a$&$(\{b,c,d\}, (b \land c) \lor (b \land d)$&$(\{e\},e)$&$1$\\
$b$&$1$&-&-\\
$c$ &$0$&-&-\\
$d$ &$(\{e\},e)$&$0$&-\\
$e$ &$(\{f\},f)$&$1$&-\\
$f$&$(\{a\},a)$&$(\{b\},b)$&$1$\\
\noalign{\smallskip}\hline
 \end{tabular}
 \end{subfigure}
 \begin{subfigure}{0.45\textwidth}
\scalebox{0.74}{
 \begin{tikzpicture}
  \node (F) at (0,0)    {\usebox{\6}};
  \node (A) at (2.5,0)    {\usebox{\1}};
  \node (E) at (5,0)   {\usebox{\5}};
  \node (G) at (2.5,-1.1)  {{\small $(b\land c) \lor (b \land d)$}};
  \node (B) at (0,-1.1)    {\usebox{\2}};
  \node (C) at (2.5,-2.2)    {\usebox{\3}};
  \node (D) at (5,-1.1)    {\usebox{\4}};

  \draw[-]
      (A) edge[]            node{}(G);

  \draw[->,>=stealth']
      (A) edge[dashed]        node {} (E)
        (D) edge              node {} (E)
        (E) edge[bend right]              node {} (F)
        (F) edge            node {} (A)
            edge[dashed]   node {} (B)
        (G) edge            node {} (B)
          edge            node {} (C)
          edge            node {} (D);
\end{tikzpicture}
}
\end{subfigure}
  \caption{Representation of the ballots (left) and the delegation structure (right) of the agents in $\Prof$ from Example~\ref{ex:unravellingprocedures}. In the graph on the right, a solid line indicates the first preference for delegation, a dashed line represents the second, and the final preference (a direct vote in $\{0,1\}$) is written next to the agents' names.}
\label{fig:UnravellingExample}
\end{figure}

\label{table:UnravellingExPreference}


\begin{itemize}
    \item [] $\unravel{\mathbf{U}}$\\
    At $\lev=1$ the procedure adds the direct votes of $b$ and $c$ to $X$. Thus, we have $X = (\Delta, 1, 0, \Delta, \Delta, \Delta)$. Then, the algorithm cannot find a direct or computable vote at $\lev = 1$, so it moves to $\lev = 2$ where it uses $Y$ to add the direct votes of $d$ and $e$, as well as $f$'s vote that is computable from $X$ by copying $b$, giving $X = (\Delta, 1, 0, 0, 1, 1)$. As no other update is possible, the algorithm sets $\lev = 1$ and it computes $a$'s vote, yielding $X = (0, 1, 0, 0, 1, 1)$, with $\cert = (1,1,1,2,2,2)$.
    \item [] $\unravel{\mathbf{DU}}$ \\
    As with $\unravel{\mathbf{U}}$, the direct votes of $b$ and $c$ are added initially, which yields to $X = (\Delta, 1, 0, \Delta, \Delta, \Delta)$, and then the algorithm moves to $\lev=2$. Unlike $\unravel{\mathbf{U}}$, the procedure $\unravel{DU}$ adds only the direct votes of $d$ and $e$, giving $X=(\Delta, 1, 0, 0, 1, \Delta)$. Returning to $\lev=1$, $a$'s vote can be computed from the votes of $b,c$ and $d$, giving  $X=(0, 1, 0, 0, 1, \Delta)$. Finally, at $\lev=1$ , $f$'s computable vote (a delegation to $a$) can be added, thus giving $X=(0, 1, 0, 0, 1, 0)$,  with certificate $\cert=(1,1,1,2,2,1)$.
    \item [] $\unravel{RU}$ \\
    First, the direct votes of $b$ and $c$ are added, each in a separate iteration, giving $X = (\Delta, 1, 0, \Delta, \Delta, \Delta)$. Then, the algorithm moves to $\lev=2$, where it chooses a single vote at random to add to $X$ from the agents $d, e$ and $f$. 
If, for example, the vote of $f$ was added, then $X=(\Delta, 1, 0, \Delta, \Delta, 1)$. At $\lev=1$, $e$'s vote can be computed from $f$'s, and then $d$'s from $e$'s, giving $X=(\Delta, 1, 0, 1, 1, 1)$. Then, at $\lev=1$, $a$'s vote can be computed from $b$, $c$ and $d$'s, yielding $X=(1, 1, 0, 1, 1, 1)$, whose certificate is $\cert=(1,1,1,1,1,2)$.
    \item [] $\unravel{DRU}$ \\
    This procedure moves as $\unravel{RU}$, except that it chooses randomly only between the direct votes of $d$ and $e$ at the iteration where $\unravel{RU}$ can also choose to select the vote of agent $f$.
\end{itemize}

Note that in this example, $\minmax$ would return outcomes corresponding to all certificates $\cert$ where $\max(\cert)=2$. This would include, e.g., $\cert = (1,1,1,2,1,1)$, which is also returned by $\minsum$, but also $\cert' = (2,1,1,2,2,2)$ and many more.
\end{example}

\subsection{Language restrictions for smart ballots}\label{sec:restrictions}

Starting from our general concept of a valid smart ballot in Definition~\ref{def:validsmartballot}, we now focus on some restrictions on the language of delegations in order to study our procedures.

We start by focusing on Boolean functions expressed as propositional formulas on a binary domain, with a few additional requirements. Firstly, we impose that the formulas are \emph{contingent}---i.e., neither a tautology, nor a contradiction---in order to avoid a direct vote in disguise for (in the case of a tautology) or against (in the case of a contradiction) the issue, as they would always evaluate to true (respectively, to false). Secondly, the formulas must be expressed in \emph{disjunctive normal form} (DNF): i.e., they are written as a disjunction of \emph{cubes}, where a cube is a conjunction of \emph{literals} (and a literal is a variable or its negation). Finally, call a cube $C$ an \emph{implicant} of formula $\varphi$ if $C \vDash \varphi$, and call $C$ a \emph{prime implicant} of $\varphi$ if $C$ is an implicant of $\varphi$ and for all other $C' \vDash \varphi$ we have that $C' \nvDash C$. Intuitively, prime implicants are the minimal partial assignments to make a formula true. A \emph{complete} DNF is the unique representation of a DNF listing all of its prime implicants.

This representation may seem restrictive, but for an implementation of our framework we could envisage a pre-processing step where the agents are aided by a computer platform when creating their ballots---which would use techniques such as the \emph{consensus method} or \emph{variable depletion} (see the textbook by \citet{crama2011boolean} for further details) to find the corresponding complete DNF of a formula.

We call this restricted language $\Blang$: \footnote{Note that in previous work \cite{colley2020smart}, the language $\Blang$ was initially defined  simply as the language of contingent propositional formulas in DNF, for which however the necessary winners cannot be computed in polynomial time. We are grateful to an anonymous reviewer for pointing this out.}

\begin{definition}[$\Blang$]\label{def:BOOLBallot}
A smart ballot $B_a$ for agent $a$ and a binary issue is in language $\Blang$ if every $F_a^h$ in $B_a$ is a contingent propositional formula in complete DNF.
\end{definition}
Observe that a propositional atom is a Boolean function corresponding to the identity function: i.e., it is equivalent to copying another agent's vote. In Example~\ref{ex:smartballot}, ballot $(v)$ does not belong to language $\Blang$ as $B_a^1 = \mathit{Maj}$ is not a Boolean formula; however, ballot $(iv)$ belongs to $\Blang$, but note that the formula $(b \land c) \lor (b \land \lnot c) \lor f$, which is equivalent to the formula used at the first preference level, would not be in $\Blang$ as it is not complete. For the language $\Blang$, we often write $\varphi^\lev_a$ instead of $F^\lev_a$. 

The following proposition shows that the necessary winner for $\Blang$ ballots can be computed in polynomial time. 
\begin{proposition}\label{prop:completenec}
 Deciding if a formula in a $\Blang$ ballot has a necessary winner can be done in polynomial time.
\end{proposition}
\begin{proof}
Observe that the necessary winner for a formula being $1$ (resp., $0$) means that the formula is true (resp., false). We first need to prove the following two claims: 
\begin{enumerate}
    \item The necessary winner of a complete DNF formula is $1$ if and only if every literal of at least one cube of the formula is true.
    \item  The necessary winner of a complete DNF formula is $0$ if and only if every cube of the formula is made false by at least one literal.
\end{enumerate}
These two claims can be computed by reading the formula and the partial truth assignment; thus, if they are true, a necessary winner can be found in polynomial time.
 
  For the right-to-left direction of claim (1), assume that one cube of the formula is true. As the formula is a complete DNF, each cube represents one of its prime implicants. By definition, if a prime implicant is made true, so is the formula. 

  For the left-to-right direction of claim (1), assume that the complete DNF formula $\varphi$ is made true by some partial truth assignment $X$. We create a cube $C$ from the partial assignment,  where if a variable $x$ is true (resp., false) in $X$ then $x$ (resp., $\lnot x$) is a literal in $C$. As $C$ is built from a partial truth assignment making $\varphi$ true, we have that $C\vDash \varphi$ and thus $C$ is an implicant of $\varphi$. Then, either $C$ is a prime implicant of $\varphi$ or there exists a prime implicant $C'$  of $\varphi$, such that $C'\vDash C$, where $C'$ contains a subset of literals in $C$. As $\varphi$ is a complete DNF, in either case there will be a cube of $\varphi$ made true by $X$ (i.e., either $C$ or $C'$).

  For the right-to-left direction of claim (2), if all cubes in the formula are made false, then the formula is also necessarily false (i.e., the necessary winner is $0$). 

  For the left-to-right direction of claim (2), assume that a complete DNF $\varphi$ evaluates to false under a partial truth assignment $X$. Yet, assume for a contradiction that there exists a cube $C$ of $\varphi$ that does not evaluate to false under $X$. As $C$ is not false, then either $C$ is true under $X$ (yielding a contradiction, as $\varphi$ would be true), or there are some variables $v\in \Var{C}$ without a truth value in $X$ and the remaining literals are made true. We can then extend $X$ for each such $v\in \Var{C}$ such that the literal of $v$ in $C$ is made true. As $\varphi$ is a complete DNF, the cube $C$ would be true---as no cube can contain contradictions (e.g., $x$ and $\neg x$). Thus, the formula $\varphi$ would be true and we would reach a contradiction.
  
  Finally, checking that each literal of at least one cube is true (or that every cube is made false by at least one literal) can be done by simply inspecting the formula together with the partial truth assignment, and thus in polynomial time.
\end{proof}

A further advantage of having delegations expressed in complete DNF is that we can check whether a ballot is valid in polynomial time:\footnote{We previously showed  \cite{colley2020smart} that checking if a ballot of contingent DNF formulas is valid is an \textsc{NP}-complete problem. Restricting formulas to contingent \emph{complete} DNFs makes this problem tractable.} a tautology in complete DNF is $\top$, a contradiction is $\bot$, and to check if two complete DNF formulas are equivalent it suffices to see if the lists of their prime implicants are the same.

The next language restriction that we introduce is to ranked liquid democracy ballots. The language $\Liqlang$ restricts the delegations to single other agents, where the delegation function is the identity function $id$. 

\begin{definition}[$\Liqlang$]\label{def:LIQBallot}
Smart ballot $B_a$ for agent $a$ belongs to $\Liqlang$ if every delegating $B_a^h$ is of the form $(\{b\}, id)$ for $b\in\N\setminus\{a\}$ and $id$ the identity function. 
\end{definition}

In some models of ranked liquid democracy---e.g., in Liquid Feedback \cite{BehrensEtAl2014}---the final backup vote must be an abstention ($*$): we denote this language as $\Liqlang_*$. 

Finally, for a given language $\mathcal{L}$ we write $\mathcal{L}[k]$ to indicate the smart ballots in $\mathcal{L}$ having at most $k$ delegations in their ordering. For instance, in Example~\ref{ex:smartballot} the smart ballots $(i)$ and $(iv)$ belong to the language $\blang{2}{}$, while ballot $(iii)$ belongs to the language $\Liqlang_*$ as well as $\liqlang{2}{}$. Note that checking if a ballot is valid for $\Liqlang$ is a tractable problem as it suffices to check that all delegation functions use $id$ and that no one delegates to themselves or to the same agent multiple times. 


\section{Computational complexity of unravellings}\label{sec:compucomplex}

In this section we study the complexity of computational problems for each of our unravelling procedures.  
First, we study how hard it is to unravel a smart profile under a given procedure. We begin with $\minsum$ and $\minmax$, showing that an associated decision problem, $\minsumbound$ and $\minmaxbound$, respectively, are \textsc{NP}-complete. However, when smart ballots are restricted to $\Liqlang$, finding a solution becomes tractable. Unlike $\minsum$ and $\minmax$, 
we show that our greedy procedures, $\unravel{\#}$ with $\mathbf{\#}\in \{\U, \DU, \RU, \DRU\}$, always terminate in a polynomial number of time steps.  


\subsection{Computational complexity of $\minsum$}\label{sec:complexopt}

In this section we study the computational complexity of finding $\minsum$ outcomes, when ballots are restricted to either the $\Blang$ or $\Liqlang$ language, finding the problem to be \textsc{NP}-complete in the former case and tractable in the latter.

We begin by studying the decision problem $\minsumbound$, 
whose input is a smart profile $\Prof$, such that every ballot is restricted to $\Blang$, and a constant $M\in \mathbb{N}$. The problem then asks if there is a consistent certificate $\cert$ that unravels $\Prof$ such that $\rank(\cert)\leq M$.
Repeatedly using $\minsumbound$ for different values of $M$ gives us the minimum bound, and a modified version of $\minsumbound$ using partial certificates would allow us to compute an outcome of $\minsum$. 
Both problems are harder than $\minsumbound$, which we now show being \textsc{NP}-complete.

\begin{lemma}\label{lem:optboundmembership}
$\minsumbound$ is in \textsc{NP}.  
\end{lemma}
\begin{proof}
Recall that $\minsumbound$ is defined on $\Blang$ profiles. We prove membership in \textsc{NP} by showing that a witness can be checked in polynomial time. 
  Our witness will be the certificate vector $\cert\in \mathbb{N}^n$, such that $\cert_i$ represents the preference level of agent $i\in \N$ when unravelling the profile. 

  First, we check that $\cert$ abides by Definition~\ref{def:certificate}: that is, for each $i\in \N$, $\cert_i$ corresponds to a preference level in $B_i$. To do this, we need to read the certificate and the profile, taking a polynomial number of time steps. Next, we check that  $\cert$ is consistent: we first find the direct voters $B_i^{\cert_i}\in \{0,1\}$ from the certificate and the profile, and we add them to a set $D$. We construct the outcome vector $X\in \{\Delta\}^n$ and append the entry $X_i=B_i^{\cert_i}$ for these direct voters, which can be done in polynomial time. Then, we check if any necessary winners can be computed from $D$: for each agent $i\in \N\setminus D$ such that there exists a $j \in D$ such that $j \in S_i^{\cert_i}$, we check if we can compute a necessary winner of $F_i^{\cert_i}$ given $X\rest_{S_i^{\cert_i}\cap D}$. If so, we add $i \in D$ and let $X_i= F_i^{c_i}(X\rest_{S_i^{c_i}\cap D})$. 
   Since all functions in the ballots are in complete DNF, by Proposition \ref{prop:completenec} we can check for a necessary winner in polynomial time. Since at least one agent gives a direct vote, we have to check at most $n-1$ agents' functions for a necessary winner in the first `round'. If the certificate $\cert$ is consistent, at least one agent is added in each round. Therefore, we have to do at most $\sum_{k=1}^{n-1}k= \frac{(n-1)n}{2}$ polynomial checks, if a single agent is found in each round. Finally, we check in polynomial time that $\sum_{i\in\N}\cert_i\leq M$.
  All steps can be done in polynomial time, showing that $\minsumbound$ is in \textsc{NP}. 
\end{proof}

We now show that $\minsumbound$ is \textsc{NP}-hard by giving a reduction from Feedback Vertex Set (\textsc{FVS}), a problem shown by \citet{karp1972reducibility} to be \textsc{NP}-complete. The input of \textsc{FVS} is a directed irreflexive graph $G=(V,E)$ and a positive integer $k$,\footnote{The formulation by \citet{karp1972reducibility} is on directed graphs $G$ which allow for reflexive edges. However, our sub-problem is also \textsc{NP}-complete, since a reduction can be given where the constructed graph $G'$ adds a dummy node $a'$ for each node $a$ that had a reflexive edge in $G$, as well as the edges $(a,a')$ and $(a',a)$.} and it asks if there is a subset $X\subseteq V$ with $|X|\leq k$ such that, when all vertices of $X$ and their adjacent edges are deleted from $G$, the remaining graph is cycle-free.

\begin{lemma}\label{lem:optboundnphard}
$\minsumbound$ is \textsc{NP}-hard. 
\end{lemma}
\begin{proof}
Recall that $\minsumbound$ is defined on the language of complete DNFs. We prove the claim by reducing from Feedback Vertex Set (\textsc{FVS}). Given an instance $(G, k)$ of \textsc{FVS}, let an instance of $\minsumbound$ be such that $\N=V$, $M=|V|+k$, and for each vertex-agent $v\in V$ their ballot $B_v$ is constructed as follows, for $O_v = \{ u \in V \mid (v,u) \in E \}$ the set of outgoing edges of vertex $v$ in $G$:
\[
B_v= (O_v, \bigwedge_{u \in O_v} u)>1.
\]

 The first delegation of each agent $v$ is a conjunction of positive literals (hence, a complete DNF), each representing one of the outgoing edges from $v$ in graph $G$. Then, the backup vote for $1$ represents the removal of the vertex $v$ in the FVS problem. For the agents $v\in V $ without any outgoing edges ($O_v=\emptyset$), their ballot is $B_v=1$.

To show the correctness of our reduction, we first prove the following claim: a graph $G$ is acyclic if and only if  $\cert=\{1\}^n$ is a consistent certificate for the profile $\Prof$ given by the translation above.

    For the left-to-right direction, we prove the contrapositive: assume that the certificate $\cert=\{1\}^n$ is not consistent for $\Prof$.  Therefore, there exists no ordering of the agents such that all their votes can be added by using the votes added previously following the ordering. This means that there exists a delegation cycle between the formulas at the first preference level of some agents in $\Prof$, as at least two agents require each others' votes to determine their own. Since by construction of $\Prof$, the literals in the formulas represent the outgoing edges in $G$, the graph $G$ is not acyclic. 

    For the right-to-left direction, let $\cert=\{1\}^n$ be a consistent certificate for $\Prof$. 
    First, note that all nodes in $G$ representing non-delegating agents in $\Prof$ have no outgoing edges.  Second, for each delegating agent $v\in V$, since they can only possibly be assigned a vote for $1$ (their backup direct vote), the truth assignment to the formula $\bigwedge_{u \in O_v} u$ can only be determined when all agents in the delegation have been assigned a vote. Hence, there can be no delegation cycles within the first preferences, as this would entail that a second preference must be used. Thus, every maximal path in $G$ starting from a node $v$ ends in a node without any outgoing edges (a node representing a direct voter). Therefore, $G$ is acyclic.
    
    We now prove the reduction using the previous claim. 
    First, assume that there exists a subset $X$ such that $|X|\leq k$ and the resulting graph has no cycles: we want to show that $\rank(\Prof)\leq M= |V|+k$. If all of the agents in $X$ receive their second preference, then all of the agents in $\N\setminus X$ get their first preference. Since this subset is acyclic, it is also consistent (given our claim above), and the addition of direct voters does not impact the consistency of a certificate. The rank of this unravelling is $|V|+|X|\leq |V|+k$ and therefore, $\rank(\Prof)\leq M= |V|+k$.


	


	Next, we show that if $\rank(\Prof)\leq M= |V|+k$, then there exists a subset $X$ such that $|X|\leq k$ and the remainder of $G$ without $X$ is cycle-free. We let $\cert$ be the certificate of unravelling $\Prof$ such that the rank is less than or equal to $|V|+k$.  From $\cert$, we build $X=\{u\mid \rank(\cert_u)=2\}$.  We remove the agents in $X$ from the profile, both their ballots and any mention of them in delegations. Note that since $\rank(\Prof)\leq M= |V|+k$, it must be the case that $|X|\leq k$. Thus, the restriction of the certificate to $v\in\N\setminus X$ must be such that $\cert_v=1$. We can now use the claim above to state that the resulting graph with nodes $V\setminus X$ is acyclic.

	
Therefore, $\minsumbound$ is \textsc{NP}-hard. 
\end{proof}

Lemmas~\ref{lem:optboundmembership} and~\ref{lem:optboundnphard} together give us the following theorem.

\begin{theorem}\label{thm:minsumNPcomp}
  $\minsumbound$ is \textsc{NP}-complete.
\end{theorem}
\begin{remark}
The reduction in the proof of Lemma~\ref{lem:optboundnphard} does not use negated literals in the ballots: thus, $\minsumbound$ would still be \textsc{NP}-complete if we were to further restrict $\Blang$ to contingent complete DNF formulas with only positive literals.
\end{remark}


The proof of our next result uses \emph{Edmonds' algorithm} \cite{edmonds1967optimum}.\footnote{Also independently suggested by \citet{chu1965shortest} and \citet{bock1971algorithm}.} This algorithm finds, for a given weighted directed graph, a minimum \emph{arborescence tree}, i.e., a directed rooted tree minimising the sum of the weights of the edges in the tree.\footnote{For undirected graphs, the corresponding problem is that of finding a minimum spanning tree.} 

Edmond's algorithm takes as input a (pre-processed) weighted directed graph $D=(V,E,w)$ and a root $r\in V$, where $V$ is a set of vertices (or nodes), $E$ is a set of edges, and $w$ is a vector of the edges' weights. 
At each step the algorithm picks a vertex $v\in V\setminus \{r\}$ that does not have yet an incoming edge in the arborescence tree and it adds an incoming edge of this vertex having minimum weight. After each edge has been added, the algorithm checks if a cycle has formed: if that is the case, the nodes involved in the cycle are \emph{contracted} to a single node $v_C$ creating a new directed graph $D'$. The algorithm continues until the contracted graph is a directed spanning tree, and then all of the contractions are expanded. 

The contraction of cycles is performed as follows. Given a set of nodes $C$ involved in a cycle, we let $V'=(V\setminus C)\cup\{v_C\}$, for $v_C$ a new node. In case $e_{uv}\in E$ for $u\notin C$ and $v\in C$, we let $e_{uv_C}\in E'$ such that $w(e_{uv_C})= w(e_{uv})-w(e_{wv})$ where $e_{wv}$ is the lowest weighted incoming edge of $v$ (the weight of $e_{uv_C}$ corresponds to the incoming weight to the cycle, minus the lowest weighted incoming weight to node $v$ in the cycle). In case $e_{vu}\in E$ for $v\in C$ and $u\notin C$, we let $e_{v_C u}\in E'$ with $w(e_{v_C u})=w(e_{vu})$. All edges whose nodes are not in the cycle $C$ remain unchanged.

\begin{theorem}\label{thm:optcomplex}
    An outcome in $\minsum(\Prof)$ on a profile $\Prof$ in $\Liqlang$ can be found in $\mathcal{O}(n(d+n))$ time, where $d$ represents the number of delegations in $\Prof$.
\end{theorem}
\begin{proof}
The proof idea is to create a graph on which to apply Edmonds' algorithm~\cite{edmonds1967optimum}. For a profile $\Prof$ of $\Liqlang$ ballots, we construct a directed graph $D= (V,E)$, where $V=\N\cup\{r\}$ for a fresh node $r$. For the edges in $E$, we let $e_{ji} \in E$ if $B_i^k =  (\{j\},j)$ for some $k$: i.e., we add an edge from $j$ to $i$ if $i$ was delegating to $j$ at $i$'s $k^\text{th}$ preference level. Furthermore, we add an edge $e_{ri} \in E$ for all $i \in \N$, representing the final direct vote of each voter. The weight of each edge is always given by the preference level of the delegation: if $B_i^k\in\Dom$ then $w(e_{ri})=k$, and if $B_i^k= (\{j\}, j)$ then $w(e_{ji})=k$.

Edmonds' algorithm returns the arborescence tree $A=(V, E')$ rooted at $r$ in time $\mathcal{O}(|V|\times|E|)$, minimising its weight $w(A)= \sum_{e_{ij}\in E'} w(e_{ij})$. Note that by applying Edmonds' algorithm to the graph $D$ above, we find an unravelling of $\Prof$, whose certificate vector $\cert$ minimises $\rank(\cert)$. Moreover, since it returns a tree which includes every node, there are no delegation cycles and every agent has one of their preference levels used: hence, the unravelling is consistent. 

  Thus, we can find a solution of $\minsum$ in $\mathcal{O}(|V|\times|E|) = \mathcal{O}((n+1)\times(d+n))$ time steps, since $|V|=n+1$ (all the agents plus the vertex $r$), and $|E|=d+n$, where $d$ represents the number of delegations in $\Prof$. To simplify the bound, this can be done in $\mathcal{O}(n(d+n))$ time steps.
\end{proof}

We can thus find an optimal unravelling of a $\Liqlang$ smart profile in a polynomial number of time steps. Furthermore, as the Edmonds' algorithm is recursive, we are guaranteed that it will terminate giving an optimal unravelling, provided that there is some tie-breaking rule when there are many optimal unravellings. We now illustrate in an example the application of Edmonds' algorithm in the proof of Theorem~\ref{thm:optcomplex}.


\begin{example}\label{ex:EdmondsAlgoLiquidopt}
We show the application of Edmonds' algorithm in the proof of Theorem~\ref{thm:optcomplex} to get a $\minsum$ outcome associated to profile $\Prof$ in Table~\ref{table:exampleEdmondsAlg}. The directed graph $D=(V,E,w)$ has nodes $V=\{a,b,c,d,e,r\}$, edges $E=\{(ra)$, $ (cb)$, $(ab)$, $(rb)$, $(dc)$, $(ec)$, $(rc)$, $(bd)$, $(ed)$, $(rd)$, $(re)\}$, and weights $w(ra) = w(cb) = w(dc) = w(bd) = w(re) = 1$, $w(ab) = w(ec) = w(ed)= 2$, and $w(rb) = w(rc) = w(rd) = 3$.
The graph $D$ is shown on the left of Figure~\ref{fig:directgraphs}, with solid, dashed, and dotted lines representing first, second, and third preference levels, respectively.

\begin{table}[h]
\centering
 \begin{tabular}{llll} 
 \hline\noalign{\smallskip}
&$B_x^1$&$B_x^2$&$B_x^3$\\
 \noalign{\smallskip}\hline\noalign{\smallskip} 
$a$&$1$&-&-\\
$b$&$(\{c\},c)$&$(\{a\},a)$&$*$\\
$c$&$(\{d\},d)$&$(\{e\},e)$&$*$\\
$d$&$(\{b\},b)$&$(\{e\},e)$&$*$\\
$e$&$0$&-&-\\
\noalign{\smallskip}\hline
\end{tabular}
  \caption{The $\Liqlang$ profile that in Example~\ref{ex:EdmondsAlgoLiquidopt} is unravelled via Edmonds' algorithm, and in  Example~\ref{ex:Liquidminmax} by the algorithm from Theorem~\ref{thm:minmaxliquid}.}
  \label{table:exampleEdmondsAlg}
\end{table}

	In Figure~\ref{fig:directgraphs}, we see at the bottom of $D$ that there is a cycle among the nodes $b,c$ and $d$, among the edges representing the first preference levels. Edmonds' algorithm contracts this cycle to a single vertex $v$, creating a second directed graph $D'=(V',E',w)$, in the centre of Figure~\ref{fig:directgraphs}. The nodes of $D'$ are $V'=\{a,v,e,r\}$; while for the edges $E'$, we keep $(ra)$ and $(re)$ but we alter the edges coming into and out of the cycle. However, note that there are only incoming edges to the cycle: $(ab), (rb), (rd), (rc), (ed), (ec)$ and thus we add to $E'$ only edges coming into $v$, taking into account the lowest weighted incoming edge to each node in the cycle.

For the edge $(rb) \in E$, we thus have an edge $(rv)\in E'$ whose weight is computed as $w(rv) = w(rb) - w(xb)$ where $w(xb)$ is the weight of the lowest incoming edge of $b$, e.g., $(cb)$, which has weight $w(cb) = 1$. Thus, $w(rv) = 3 - 1 = 2$, and analogously for $(rc)$ and $(rd)$. For the edge $(ed) \in E$, we have an edge $(ev)\in E'$ whose weight is $w(ev) = w(ed) - w(xd) = 2 - 1 = 1$, and similarly for $(ec)$. Finally, for $(ab) \in E$, we have an edge $(av)\in E'$ with weight $w(av) = w(ab) - w(xb) = 1$.

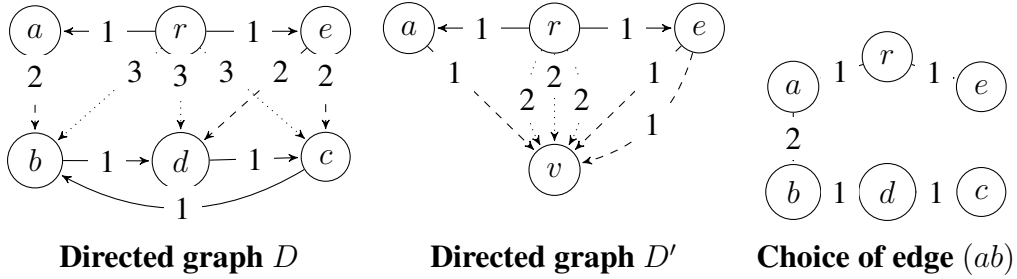
\begin{figure}[t]
\begin{center}
\begin{tabular}{ccc}
	\begin{tikzpicture}[node distance=1.2cm]

  \node[draw, circle]     (r)                        {$r$};
     \node[draw, circle] (f) [right= 1.25cm of r]   {$e$};
  \node[draw, circle]     (a) [left= 1.25cm of r]      {$a$};
  \node[draw, circle]      (b) [below= 1cm of a]      {$b$};
  \node[draw, circle]       (d) [below= 1cm of r]      {$d$};
  \node[draw, circle]             (c) [below= 1cm of f]      {$c$};

\draw[->, >=stealth',shorten >=1pt] 
(r) -- (a) node [midway,fill=white] {1} ;
\draw[->, >=stealth',shorten >=1pt]
(r) -- (f) node [midway,fill=white] {1} ;
\draw[->, >=stealth',shorten >=1pt, dashed]
(a) -- (b) node [near start,fill=white] {2} ;
\draw[->, >=stealth',shorten >=1pt]
(b) -- (d) node [midway,fill=white] {1} ;
\draw[->, >=stealth',shorten >=1pt]
(d) -- (c) node [midway,fill=white] {1} ;
\draw[->, >=stealth',shorten >=1pt, dotted]
(r) -- (b) node [near start,fill=white] {3} ;
\draw[->, >=stealth',shorten >=1pt, dotted]
(r) -- (c) node [near start, fill=white] {3} ;
\draw[->, >=stealth',shorten >=1pt, dotted]
(r) -- (d) node [near start,fill=white] {3} ;
\draw[->, >=stealth',shorten >=1pt, dashed]
(f) -- (d) node [near start,fill=white] {2} ;
\draw[->, >=stealth',shorten >=1pt, dashed]
(f) -- (c) node [near start,fill=white] {2} ;
\draw[->, >=stealth',shorten >=1pt]
(c) to [bend left] node [midway,fill=white] {1} (b)  ;
\end{tikzpicture}
& 
\begin{tikzpicture}[node distance=1.2cm]
  \node[draw, circle]     (r)      {$r$};
     \node[draw, circle]     (f) [right= 1.25cm of r]    {$e$};
  \node[draw, circle]     (a) [left= 1.25cm of r]      {$a$};
   \node[draw, circle]       (v) [below= 1.2cm of r]      {$v$};
   \node (x) [below=0.15 of v] {};
\draw[->, >=stealth',shorten >=1pt]
(r) -- (a) node [midway,fill=white] {1} ;
\draw[->, >=stealth',shorten >=1pt]
(r) -- (f) node [midway,fill=white] {1} ;
\draw[->, >=stealth',shorten >=1pt, dashed]
(a) -- (v) node [near start,fill=white] {1} ;
\draw[->, >=stealth',shorten >=1pt, dashed]
(f) -- (v) node [near start,fill=white] {1} ;
\draw[->, >=stealth',shorten >=1pt,dotted]
(r) to [bend left] node [midway,fill=white] {2} (v)  ;
\draw[->, >=stealth',shorten >=1pt,dotted]
(r) to [bend right] node [midway,fill=white] {2} (v)  ;
\draw[->, >=stealth',shorten >=1pt,dashed]
(f) to [bend left] node [midway,fill=white] {1} (v)  ;
\draw[->, >=stealth',shorten >=1pt, dotted]
(r) -- (v) node [near start,fill=white] {2} ;
\end{tikzpicture}
&
\begin{tikzpicture}
  \node[draw, circle]     (r)  at (0,-1.35)      {$r$};
  \node[draw, circle]     (a) at (-1.25, -1.75)  {$a$};
  \node[draw, circle]     (f) at (1.25, -1.75)     {$e$};
  \node[draw, circle]      (b) at (-1.25, -3.15)       {$b$};
  \node[draw, circle]      (d) at (0, -3.15)     {$d$};
  \node[draw, circle]      (c) at (1.25, -3.15)      {$c$};
  \node (x) at (0, -3.5) {};
   \draw[->, >=stealth',shorten >=1pt]
(r) -- (a) node [midway,fill=white] {1} ;
   \draw[->, >=stealth',shorten >=1pt]
(r) -- (f) node [midway,fill=white] {1} ;
   \draw[->, >=stealth',shorten >=1pt, dashed]
(a) -- (b) node [midway,fill=white] {2} ;
   \draw[->, >=stealth',shorten >=1pt]
(b) -- (d) node [midway,fill=white] {1} ;
   \draw[->, >=stealth',shorten >=1pt]
(d) -- (c) node [midway,fill=white] {1} ;
\end{tikzpicture}
\\
\textbf{Directed graph $D$}& \textbf{Directed graph $D'$}&\textbf{Choice of edge $(ab)$}\\
\end{tabular}
\end{center}
\caption{Three stages of unravelling the $\Liqlang$ profile $\Prof$ from Table~\ref{table:exampleEdmondsAlg} by using Edmonds' algorithm. The directed graph $D$ (left) represents the initial profile. In $D'$ (centre), the nodes $b,c$ and $d$ are contracted into $v$, as they were in a cycle in $D$. The arborescence tree (right) is the output where the edge $(ab)$ was chosen to break the tie, and it corresponds to an outcome of $\minsum$ on $\Prof$.
}\label{fig:directgraphs}
\end{figure}

Since there are no cycles in $D'$, we can find an arborescence tree of $D'$ rooted at $r$ with edges $(ra),(re)$ and then either $(av)$ or $(ev)$, as they both have the lowest weight of $1$. Suppose that $(av)$ is chosen: this represents the delegation from $b$ to $a$ with weight $2$ and in the arborescence tree $a$ will be followed by $b$---this unravelling in shown on the right hand-side of Figure~\ref{fig:directgraphs}. From here the unravelling continues, until all of the vertices of the cycle have been chosen; giving the edges $(bd)$ and $(dc)$. 
Alternatively, the algorithm could have chosen the edges $(ec)$ or $(ed)$ instead of $(ab)$: all of these unravellings are optimal, with a total weight of $6$. 

\end{example}

\subsection{Computational complexity of $\minmax$}\label{subsec:compminmax}

We study here the computational complexity of the $\minmax$ rule, showing that like $\minsum$: it is \textsc{NP}-hard for the language $\Blang$, and tractable for $\Liqlang$.
We begin by studying the problem $\minmaxbound$, which takes as input a valid smart profile $\Prof$ restricted to $\Blang$ as well as a constant $M$, and it asks whether there is an unravelling given by a certificate $\cert$ such that $\max_{a\in \N}(\cert)\leq M$. We first show membership in \textsc{NP} and then \textsc{NP}-hardness. 

\begin{lemma}\label{lem:minmaxboundmembership}
  $\minmaxbound$ is in \textsc{NP}.
\end{lemma} 
\begin{proof}
Recall that $\minmaxbound$ is defined on $\Blang$ profiles. To prove membership in \textsc{NP} we can check in polynomial time that a certificate vector $\cert$ abides by Definition~\ref{def:certificate} and is consistent, as we did for Lemma~\ref{lem:optboundmembership}.
  Then, we need to check that all entries in the certificate are less than or equal to the constant $M$, which can be done in polynomial time. 
\end{proof}

\begin{lemma}\label{lem:minmaxboundhard}
  $\minmaxbound$ is \textsc{NP}-hard.
\end{lemma}
\begin{proof}
Recall that $\minmaxbound$ is defined on the language of $\Blang$ where delegations are expressed as complete DNFs.    We reduce from the \textsc{NP}-complete problem \textsc{CNF-Sat} which has as input a formula $\varphi$ in CNF---i.e., a conjunction of \emph{clauses} (disjunctions of literals)---and it asks if there exists a satisfying assignment for $\varphi$.   

For a given formula $\varphi$ in CNF, let $C=\{c\mid c\text{ is a clause of }\varphi\}$ be a set of variables $c$, each one representing one of the clauses of $\varphi$. Construct now an instance of $\minsumbound$ where $M=2$ and the set of agents is $\N=\{x,y\}\cup C\cup \Var{\varphi}$, for $x$ and $y$ fresh variables. The ballots are defined as follows:

    \begin{itemize}
        \item $B_x=(1)$,
        \item $B_v=((\{x\},x)>0)$ for all $v\in\Var{\varphi}$,
        \item $B_y=((\{x\}\cup C, x\wedge\bigwedge_{c\in C}c)>(C, \bigwedge_{c\in C}c)> 1)$,
        \item $B_c= ((\{y\},y)> (\{y\}\cup\Var{c},y\vee \bigvee_{l\in c}l)> 1)$ for all $c\in C$, \\ where $l\in c$ represents the literal $l$ of clause $c$, and if $c$ contains a variable and its negation (i.e., $c$ is a tautology), we remove the second delegation. 
    \end{itemize}
    
    Note that each delegation is a complete DNF since it is either a cube or a clause. 
    
    Assume that $\varphi$ is satisfiable. Then, each agent $v\in \Var{\varphi}$ gets their first preference if in the satisfying truth assignment of $\varphi$ the variable $v$ is true ($\cert_v=1$), and their second preference if $v$ is false ($\cert_v=2$). The satisfying truth assignment of $\varphi$ makes every clause $c\in C$ true, and therefore one literal in $y\vee\bigvee_{l\in c}l$ is made true, making the whole formula true. Thus, agents $c\in C$ cannot receive higher than their second preference ($\cert_c\leq 2$). Agent $y$ can receive their first preference ($\cert_y=1$), given that we can determine the vote of each $c\in C$. Finally, agent $x$ receives their first preference ($\cert_x=1$). Therefore, if $\varphi$ is satisfiable, there is an unravelling such that every agent receives at most their second preference level in the ballot. 
    
    Next, assume that $\varphi$ is not satisfiable. For a contradiction, assume that there exists an unravelling with certificate $\cert$ such that $\max(\cert)\leq 2$. Since $\cert_x=1$ and for all $v \in \Var{\varphi}$ we have $\cert_v\leq 2$, it must be the case that either $\cert_y=1$ or $\cert_y=2$. 
    
    If $\cert_y=1$ then either all clauses $c\in C$ evaluate to $1$  or there exists a $c\in C$ whose vote is $0$. In the latter case, this has to come from $c$'s first or second preference. It cannot be $c$'s first delegation, as in this scenario $y$'s vote is determined by $c$'s and thus the unravelling would not be consistent. Their second preference can only be $0$ if all of the literals of $c$ and $y$ are false, which without the vote of $y$ cannot be determined: thus, we have reached a contradiction.  However, if the votes of all of $c\in C$ are $1$ this either means that each $c\in C$ can be made true (using the second preference delegation) and therefore $\varphi$ is satisfiable, or the third preference of $c\in C$ have been used and therefore, $\max(\cert)> 2$. In both cases we reach a contradiction. 
    
    The same reasoning holds for $\cert_y=2$, and thus this concludes the proof. \end{proof}

Lemmas~\ref{lem:minmaxboundmembership} and~\ref{lem:minmaxboundhard} together give us the following:

\begin{theorem}\label{thm:minmaxNP}
  $\minmaxbound$ is \textsc{NP}-complete.
\end{theorem}

Next we study the complexity of finding a $\minmax$ solution on $\Liqlang$ ballots. 


\begin{theorem}\label{thm:minmaxliquid}
  An outcome of $\minmax$ on a profile $\Prof$ in $\Liqlang$ can be found in time $\mathcal{O}(n^2\ell^2)$, where $\ell$ is the highest preference level of any agent in the profile.
\end{theorem}
\begin{proof}

We provide an algorithm to find a $\minmax$ outcome, by transforming the profile $\Prof$ into a directed graph and then finding an arborescence tree. 

Construct a directed graph $G=(V,E,w)$ where $V=\N\cup\{r\}$, $G$ is rooted at $r$, and the set of edges $E$ with weights $w$ is constructed iteratively from $E = \emptyset$. Starting from $\lev = 1$ until $\lev = \ell$, where $\ell$ is the maximum preference level given by any agent in $\Prof$, the following procedure is executed:
  \begin{enumerate}
    \item Add to the current set $E$ an edge $e_{ij}$ if $B^\lev_i=(\{j\}, j)$, and an edge $e_{ri}$ if $B_i^\lev\in \Dom$. Namely, $E:= E\cup \{e_{ji}\mid B^\lev_i=(\{j\}, j)\}\cup \{e_{ri}\mid B_i^\lev\in \Dom\}$. In both cases, let $w(e_{ij}) = w(e_{ri}) = \lev$;
    \item Check in $\mathcal{O}(|V|+|E|)$ time (see, e.g.,  \cite{cormen2009introduction}, pg. 606) if there is a path from $r$ to $a$ via $E$, for $a\in \N$ (hence, this step has to be repeated $n$ times, for each $a \in \N$). Then, if $r$ is connected to all $a\in \N$, exit the loop; otherwise, if there is some $a\in \N$ not connected to $r$, let $\lev:=\lev+1$.
  \end{enumerate}
  
  After the execution of this iterative procedure, we thus obtain a graph $G$ where all nodes are connected to $r$. Then, we can find \emph{any} arborescence tree from $E$ rooted at $r$ in $\mathcal{O}(|V|+|E|)$ time \cite[pg. 19]{kozen2012design}. The certificate $\cert$ for an outcome of $\minmax$ is then given by $\cert_a$ being the incoming weight of $a$ in the arborescence tree, for all $a\in\N$. Intuitively, we obtain a $\minmax$ outcome since the root $r$ represents direct votes, if there are paths from $r$ to any agent we can determine their votes, and since the edges are added iteratively we know that a path does not exist for a lower preference level.

As all agents give a backup vote, there will eventually be $e_{ri}\in E$ to add for $i\in \N$ and thus the algorithm always terminates. Since the loop iterates at most $\ell$ times, and each time it makes $n$ checks, each bounded by $\mathcal{O}(|V|+|E|)$, it overall takes at most $\mathcal{O}(n\ell (|V|+|E|))$. Since $|V|=n+1$ and $|E|\leq n\ell$, the time bound is $\mathcal{O}(n\ell(n+1+n\ell))$. In $\mathcal{O}(|V|+|E|)=\mathcal{O}(n+1+n\ell)$ an arborescence tree is found. 
Thus, a solution can be found in $\mathcal{O}((n\ell+1)(n+1+n\ell))$ time steps, which can be simplified to $\mathcal{O}(n^2\ell^2)$.
\end{proof}

We now show in an example the application of the algorithm from the proof of Theorem~\ref{thm:minmaxliquid} to find a $\minmax$ outcome on a $\Liqlang$ profile.

\begin{example}\label{ex:Liquidminmax}
	
	Consider the $\Liqlang$ profile in Table~\ref{table:exampleEdmondsAlg}. We construct the directed graph $D_1=(V,E_1,w)$, shown in Figure~\ref{fig:directgraphsminmaxExample} (left), where $V=\{a,b,c,d,e,r\}$ and $E_1$ are the edges added when considering $\lev = 1$. Since the nodes $b,d$ and $c$ are not connected to the root $r$ in $D_1$, we set $\lev = 2$ and we create the graph $D_2=(V,E_2,w)$, shown in Figure~\ref{fig:directgraphsminmaxExample} (right). The set $E_2$ thus contains edges representing all of the first and second preference levels. Since in $D_2$ there is a path from $r$ to every other node, we search for an arborescence tree that will represent a $\minmax$ outcome, e.g., via a depth-first algorithm. One such tree has edges $\{(ra)$, $(re)$, $(ab)$, $(bd)$, $(dc)\}$.

	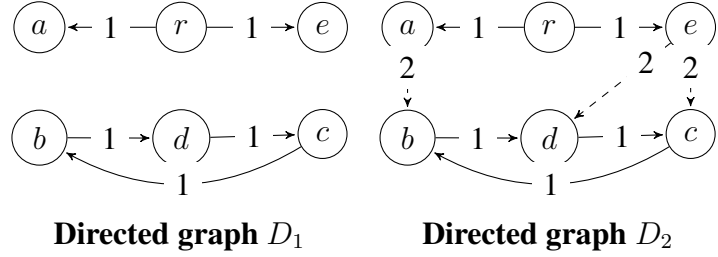
\begin{figure}[h]
\begin{center}
\begin{tabular}{cc}
	\begin{tikzpicture}[node distance=1.2cm]

  \node[draw, circle]     (r)                        {$r$};
     \node[draw, circle] (f) [right= 1.2cm of r]   {$e$};
  \node[draw, circle]     (a) [left= 1.2cm of r]      {$a$};
  \node[draw, circle]      (b) [below= 0.75cm of a]      {$b$};
  \node[draw, circle]       (d) [below= 0.75cm of r]      {$d$};
  \node[draw, circle]             (c) [below= 0.75cm of f]      {$c$};

\draw[->, >=stealth',shorten >=1pt] 
(r) -- (a) node [midway,fill=white] {1} ;
\draw[->, >=stealth',shorten >=1pt]
(r) -- (f) node [midway,fill=white] {1} ;
\draw[->, >=stealth',shorten >=1pt]
(b) -- (d) node [midway,fill=white] {1} ;
\draw[->, >=stealth',shorten >=1pt]
(d) -- (c) node [midway,fill=white] {1} ;
\draw[->, >=stealth',shorten >=1pt]
(c) to [bend left] node [midway,fill=white] {1} (b)  ;
\end{tikzpicture}
& 
	\begin{tikzpicture}[node distance=1.2cm]

  \node[draw, circle]     (r)                        {$r$};
     \node[draw, circle] (f) [right= 1.2cm of r]   {$e$};
  \node[draw, circle]     (a) [left= 1.2cm of r]      {$a$};
  \node[draw, circle]      (b) [below= 0.75cm of a]      {$b$};
  \node[draw, circle]       (d) [below= 0.75cm of r]      {$d$};
  \node[draw, circle]             (c) [below= 0.75cm of f]      {$c$};

\draw[->, >=stealth',shorten >=1pt] 
(r) -- (a) node [midway,fill=white] {1} ;
\draw[->, >=stealth',shorten >=1pt]
(r) -- (f) node [midway,fill=white] {1} ;
\draw[->, >=stealth',shorten >=1pt, dashed]
(a) -- (b) node [near start,fill=white] {2} ;
\draw[->, >=stealth',shorten >=1pt]
(b) -- (d) node [midway,fill=white] {1} ;
\draw[->, >=stealth',shorten >=1pt]
(d) -- (c) node [midway,fill=white] {1} ;
\draw[->, >=stealth',shorten >=1pt, dashed]
(f) -- (d) node [near start,fill=white] {2} ;
\draw[->, >=stealth',shorten >=1pt, dashed]
(f) -- (c) node [near start,fill=white] {2} ;
\draw[->, >=stealth',shorten >=1pt]
(c) to [bend left] node [midway,fill=white] {1} (b)  ;
\end{tikzpicture}

\\
\textbf{Directed graph $D_1$}& \textbf{Directed graph $D_2$}\\
\end{tabular}
\end{center}
\caption{Application of the algorithm in the proof of Theorem~\ref{thm:minmaxliquid} to the $\Liqlang$ profile $\Prof$ from Table~\ref{table:exampleEdmondsAlg}. The directed graph $D_1$ (left) shows the first iteration of the algorithm for the first preference levels of the agents. As in $D_1$ there is not a path from $r$ to every other node, the algorithm moves to the second iteration, constructing $D_2$ (right), which shows the agents' first and second preference levels. Since $D_2$ is connected, the algorithm terminates.  }\label{fig:directgraphsminmaxExample}
\end{figure}
\end{example}
\subsection{Computational complexity of greedy unravellings}\label{subsec:terminationGreedy}

We show here that $\unravel{\#}$ always terminates when paired with any update procedure $\#\in\{\mathbf{U}, \mathbf{DU},\mathbf{RU}, \mathbf{DRU}\}$, given a valid profile. Next, we show that they are all tractable algorithms, terminating in a polynomial number of time steps.

\begin{proposition}\label{prop:unraveltemrination}
Algorithms $\unravel{\#}$ with $\#\in \{\mathbf{U},\mathbf{DU},\mathbf{RU},\mathbf{DRU}\}$ always terminate on a valid smart profile $\Prof$.
\end{proposition}

\begin{proof}

Let $\Prof$ be a valid smart profile for $n$ agents. For the sake of a contradiction, assume that $\unravel{\#}$ by Algorithm~\ref{Alg:Gen} does not terminate on $\Prof$. Hence, \textsc{Unravel} cannot exit the {\bf while} loop from either line~\ref{2nd:line:repeat2}, due to no direct votes being computable at any preference level, or from line~\ref{1st:line:repeat1}, due to $X\notin \Dom^n$. 

Consider \textsc{Unravel} being unable to terminate due to a cycle involving the {\bf while} loop from line~\ref{2nd:line:repeat2}. Let $A = \{a \in \N \mid x_a = \Delta\}$ be the set of agents whose votes have not been computed due to a cycle. As $\Prof$ is a valid smart profile, we know that for all $a \in A$, $B_a$ has a finite number of preference levels\footnote{Recall that since both $\Dom$ and the possible sets of delegates are finite, and since all functions given in an agent's valid ballot must differ, the possible number of functions must also be finite.} and the final preference is a direct vote. In each of the update procedures ($\mathbf{U}, \mathbf{DU}, \mathbf{RU}$ and $\mathbf{DRU}$), after a finite number of loops, we will reach a direct vote of an agent in $A$. Each of the update procedures will add at least one direct vote to $X$ at this point, breaking this cycle. Moreover, no procedure replaces a vote in $X$ with $\Delta$ or with any value not in $\Dom$.

Therefore, if the algorithm does not terminate,  it must be due to the {\bf while} loop at line~\ref{1st:line:repeat1}. This can only happen while $X\notin \Dom^n$. However, as we can exit the cycle from line~\ref{2nd:line:repeat2}, the algorithm always changes some $x_a=\Delta$ to a vote in $\Dom$. Thus, after a finite number of iterations we will have that $X\in \Dom^n$ and \textsc{Unravel} terminates.
\end{proof}

Next, we show that our unravelling procedures terminate in polynomial time on $\Blang$ ballots. Recall that the delegations in a $\Blang$ ballot are Boolean functions $\varphi$ expressed in complete DNF. 
 The size of the input for $\unravel{\#}$ for smart ballots in $\Blang$ is in $\mathcal{O}(\max_p(\Prof)\cdot n\cdot \max_\varphi(\Prof))$, where $\max_p(\Prof)$ is the highest preference level of any ballot in $\Prof$ and $ \max_\varphi(\Prof)$ is the maximum length of any formula in $\Prof$.

\begin{proposition}\label{prop:polynomialterminationunravel}
$\unravel{\#}$ for $\# \in  \{\mathbf{U},\mathbf{DU},\mathbf{RU},\mathbf{DRU}\}$ terminates in at most $\mathcal{O}(n^2\cdot \max_p(\Prof)\cdot  \max_\varphi(\Prof))$ time steps, on a valid smart profile $\Prof$ in $\Blang$.
\end{proposition}

\begin{proof}
The {\bf while} loop from line~\ref{1st:line:repeat1} in \textsc{Unravel} (see Algorithm~\ref{Alg:Gen}) can be repeated at most $n$ times (when a single  vote is added to $X$ at each iteration). Moreover, the {\bf while} loop from line~\ref{2nd:line:repeat2} can be repeated at most $\max_p(\Prof)$ times, when all smart ballots are of the same length and no vote is computable in the first $\max_p(\Prof) - 1$ iterations. 
  
 The following is executed at most $n \cdot \max_p(\Prof)$ times.  $\unravel{\#}$ checks that for each agent $a$ such that $x_a =\Delta$ (at most $n-1$) either $B_{a}^\texttt{lev}\in \Dom$ or $\varphi^\texttt{lev}_{a}$ has a necessary winner (depending on the update procedure used). As each $\varphi^\texttt{lev}_{a}$ is a complete DNF, to verify if it has a necessary winner we check if either $(i)$ all literals of a cube of $\varphi^\texttt{lev}_{a}$ are made true by $X\rest_{ S^\texttt{lev}_{a}}$, or $(ii)$ one literal in each cube is made false by $X\rest_{ S^\texttt{lev}_{a}}$, returning a direct vote of $1$ or $0$, respectively, as described in Proposition~\ref{prop:completenec}.

 The use of $\unravel{\#}$ takes at most $\mathcal{O}(n \cdot 2 \max_\varphi(\Prof))$ steps, which is equivalent to  $\mathcal{O}(n \cdot \max_\varphi(\Prof))$ steps. Thus, $\unravel{\#}$ with $\#\in\{\mathbf{U},\mathbf{DU},\mathbf{RU},\mathbf{DRU}\}$ yields a vector $X$ of direct votes in $\mathcal{O}(n^2\cdot \max_p(\Prof)\cdot  \max_\varphi(\Prof))$ time steps.
  \end{proof}

\section{Comparing the unravelling procedures}\label{sec:comparingunravelling}

In this section we complement the results of Section~\ref{sec:compucomplex}, which analysed the computational complexity of our unravelling procedures, with the aim to further distinguish our defined unravellings, to understand when a procedure would be preferable. 




\subsection{Restrictions yielding distinct or identical outcomes}
\label{subsec:compareunravelling}

We study here under which restrictions on the language of the ballots, the outcomes of our unravelling procedures coincide or differ. First, we show that all unravelling procedures defined in Subsection~\ref{subsec:unravel} can give different outcomes, even when the ballots are restricted to $\Liqlang$.

\begin{proposition}\label{prop:unravelldifferent}
The unravellings $\unravel{\#}$, for $\# \in  \{\mathbf{U},\mathbf{DU},\mathbf{RU},\mathbf{DRU}\}$, $\minsum$ and $\minmax$ can give different certificates and outcomes on the same smart profile $\Prof$ of $\Liqlang$ ballots.
\end{proposition}

\begin{proof}
Consider the $\Liqlang$ profile $\Prof$ for the domain $\Dom=\{1,0\}$ and the set of agents $\N=\{a,b,c,d\}$ presented in Table~\ref{table:UnravellingdifferentOutcomes} (left). The outcomes of the unravelling procedures and their certificates are also shown in Table~\ref{table:UnravellingdifferentOutcomes} (right). However, we do not show the outcomes of $\minmax$, since it returns all consistent certificates such that no entry is greater than $3$; e.g., it will also include the certificate $\cert= (3,3,2,2)$ giving the outcome $( 1, 0, 0,1)$. Since the latter is not an outcome of any of the other procedures, $\minmax$ differs from those. Moreover, while procedures $\unravel{\RU}$ and $\unravel{\DRU}$ give the same outcomes $(1,1,1,1)$ and $(0,0,0,1)$, they are returned at different rates.
\end{proof}
\begin{table}[h]
\centering
  \begin{tabular}{lllll} 
 \hline\noalign{\smallskip}
&$B_x^1$&$B_x^2$&$B_x^3$&$B_x^4$\\
 \noalign{\smallskip}\hline\noalign{\smallskip} 
$a$&$(\{b\}, b)$&$(\{c\},c)$&$(\{d\},d)$&$1$\\
$b$&$(\{a\}, a)$&$(\{c\},c)$&$0$&-\\
$c$&$(\{a\}, a)$&$(\{b\},b)$&$1$&-\\
$d$&$(\{a\},a)$&$1$&-&-\\
\noalign{\smallskip}\hline
 \end{tabular}
 \quad
 \begin{tabular}{lll} 
 \hline\noalign{\smallskip}
Procedure&Outcome& Certificate\\
 \noalign{\smallskip}\hline\noalign{\smallskip} 
$\U$&$(1,0,1,1)$&$(3,3,3,2)$\\
\noalign{\smallskip}\hline\noalign{\smallskip} 
$\DU$& $(0,0,1,1)$& $(1,3,3,2)$\\
\noalign{\smallskip}\hline\noalign{\smallskip} 
$\RU$&$(1,1,1,1)$&$(3,1,1,2)$\\                &$(0,0,0,1)$&$(1,3,1,2)$\\
&$(1,1,1,1)$&$(2,1,3,2)$\\
&$(1,1,1,1)$&$(1,2,3,2)$\\
\noalign{\smallskip}\hline\noalign{\smallskip} 
$\DRU$ 
&$(0,0,0,1)$&$(1,3,1,2)$\\
&$(1,1,1,1)$&$(2,1,3,2)$\\
&$(1,1,1,1)$&$(1,2,3,2)$\\
\noalign{\smallskip}\hline\noalign{\smallskip} 
$\minsum$&$(0,0,0,0)$&$(1,3,1,1)$\\
\noalign{\smallskip}\hline
 \end{tabular}
 \caption{On the left, we show the profile $\Prof$ used in the proof of Proposition~\ref{prop:unravelldifferent}. On the right, the table shows the outcomes and certificates of unravelling profile $\Prof$ with the procedures $\unravel{\U}$, $\unravel{\DU}$, $\unravel{\RU}$, $\unravel{\DRU}$, and $\minsum$.}
\label{table:UnravellingdifferentOutcomes}
\end{table}

We now show that when we restrict the ballots to $\liqlang{1}{*}$ the outcome is the the same for all our unravelling procedures, except for $\minmax$.

\begin{proposition}\label{prop:liquiddiffer}
If $\Prof \in\liqlang{1}{*}$, the procedures $\minsum$ and $\textsc{Unravel}(\mathbf{\#})$ for $\mathbf{\#}\in \{\mathbf{U},\mathbf{DU},\mathbf{RU},\mathbf{DRU}\}$ give the same outcome $X$, but the certificate may differ. 
\end{proposition}

\begin{proof}

$\textsc{Unravel}(\mathbf{U})$ and $\textsc{Unravel}(\mathbf{DU})$ act in an identical manner for $\liqlang{1}{*}$ ballots. They first add all non-delegating agents' votes to $X$. Then, they iteratively unravel the first preference delegations of all agents who are not in a cycle. Once no more votes can be added from the first preference level, i.e., there are agents in a delegation cycle, these agents are assigned their second choice, i.e., the abstention $*$.



$\unravel{\RU}$ picks one agent at a time from the first preference level who either gives a direct vote, or their delegate has a vote in $X$. When no more agents are available at the first preference level, the remaining agents are in delegation cycles. Moving to the second preference level, one of these agents will be added with an abstention. Consequently, everyone caught in this delegation cycle will also receive abstentions from the agent who was picked at random. This is repeated until all cycles have been resolved and all agents have a vote in $X$. 

$\textsc{Unravel}(\mathbf{DRU})$ first adds one by one the direct votes of the agents who do not delegate. Then, it does the same for delegating agents whose delegate already has a vote in $X$. Once there are no more agents to add from their top preference, the procedure adds a single random agent with an abstention (from their second preference) and then it continues as for \textsc{Unravel($\mathbf{RU}$)}, until all delegation cycles are resolved.

 Finally, $\minsum$ returns all outcomes that minimise the total rank. Therefore, all agents will receive their first preference, except for a single agent from each delegation cycle, as in the previous unravellings. Observe that on any profile $\Prof$ in $\liqlang{1}{*}$, we have $\Crest{\minsum}{\Prof}=\Crest{\unravel{\RU}}{\Prof}=\Crest{\unravel{\DRU}}{\Prof}$.
\end{proof}

\begin{remark}
  All our six unravelling procedures will have the certificate $\cert=\{1\}^n$ on $\liqlang{1}{*}$ profiles with no delegation cycles. However, if there are cycles, $\minmax$ will return many outcomes---including the one whose certificate gives to all delegating agents their second preference $(*)$, regardless of if they are in a delegation cycle or not. Furthermore, Proposition~\ref{prop:liquiddiffer} does not hold for $\liqlang{1}{}$, where backup votes are not restricted to $*$, as the tie-breaking affects the outcome $X$.
\end{remark}

\begin{remark}\label{remark:breadth+depth}
  The breadth-first and depth-first rules by \citet{kotsialou2018breadth} differ from all six of our unravellings. 
  Consider $\N=\{a,b,c\}$, an issue with domain $\Dom=\{1,0,*\}$, and agents' ballots as follows: $B_a = ((\{b\}, b)> (\{c\}, c)> *)$, $B_b=(*)$, and $B_c=(1)$. Our six unravelling procedures would return the outcome $(*,*,1)$ with certificate $\cert=(1,1,1)$, whereas the breadth-first and depth-first procedures would return the outcome $(1,*,1)$ with certificate $\cert=(2,1,1)$. 
\end{remark}

Next, we show that all possible outcomes of $\unravel{\DRU}$ are also possible outcomes of $\unravel{\RU}$, as the set of certificates of the former is a subset of the set of certificates of the latter. 
\begin{proposition}\label{prop:DRUsubsetRU}
  $\Crest{\unravel{\DRU}}{\Prof}\subseteq \Crest{\unravel{\RU}}{\Prof}$ for valid smart profiles $\Prof$.
\end{proposition}
\begin{proof}
  At any iteration of $\unravel{\RU}$, the random choice can either include direct voters or not. If there are direct voters, $\unravel{\DRU}$ will thus have a subset of the choices of $\unravel{\RU}$ (and hence potential outcomes). If there are no direct voters at an iteration, the potential outcomes from this step are the same for $\unravel{\RU}$ and $\unravel{\DRU}$. Thus, all certificates of $\unravel{\DRU}$ will also be certificates of $\unravel{\RU}$. 
\end{proof}

\subsection{Participation axioms }

In this subsection we study two properties of resolute unravelling procedures, focusing on a binary domain (with abstentions) $\Dom=\{0,1,*\}$. Both properties were proposed by \citet{kotsialou2018breadth} and they focus on a voter's incentive to participate in the election, either by voting directly or by delegating, in line with the classical participation axiom from social choice (see, e.g., \cite{Moulin1988}).

We assume that an agent $a$ expressing a direct vote for  $x \in \{0,1\}$ prefers $x$ over $1-x$ and over an abstention, denoted by $x >_a 1-x$ and $x >_a *$, respectively. Furthermore, we focus on resolute rules to directly compare the breadth-first and depth-first procedures to our own, as the participation axioms were originally  constructed to study these procedures.

First, we make a distinction between the unravelling procedures being \emph{resolute}, as our greedy procedures, or \emph{irresolute}, as our optimal procedures. By the former, a unique outcome is returned, and by the latter, possibly many tied outcomes are returned. Although throughout the paper we present \emph{all} outcomes of the greedy procedures with random voter selection $\RU$ and $\DRU$ (see, e.g., the outcomes displayed in Table~\ref{table:UnravellingdifferentOutcomes}), these procedures are resolute as defined in Algorithms~\ref{Alg:RU} and~\ref{Alg:DRU}.


\begin{definition}[Cast-Participation]\label{def:CastPart}
A resolute voting rule $r$ and a resolute unravelling procedure $\mathcal{U}$ satisfy \emph{cast-participation} if 
for all valid smart profiles $\Prof$ and agents $a \in \N$ such that $B_{a} \in \Dom\setminus\{*\}$ we have for all $B_{a}'\neq B_{a}$
  $$r(\mathcal{U}(\Prof)) \geq_a r(\mathcal{U}(\Prof_{-a}, B_{a}')) $$
  where $\Prof_{-a}$ is equal to $\Prof$ without $a$'s ballot. For randomised procedures we require the inequality to hold for any possible outcome of $\mathcal{U}$.
\end{definition}


Cast-participation implies that agents who vote directly have an incentive to do so, rather than to express any other ballot.  In order to prove if a pair of an unravelling procedure and an aggregation rule satisfies such a participation axiom, we need some further notation.
%
Let the set of voters \emph{influenced} by a voter $a$ in a profile $\Prof$ using a resolute deterministic unravelling procedure $\mathcal{U}$ be $I^\mathcal{U}(\Prof, a)=\{ b \mid a \in~S_{b}^k$ $\text{for } \mathcal{U}(\Prof)=X_\cert \text{ with } \cert \in \C{\Prof} \text{ and } \cert_b=k
\}$.
%
Further,  let $I^\mathcal{U}_*(\Prof, a) = I^\mathcal{U}(\Prof,a) \cup \{c  \mid  c \in I^\mathcal{U}(\Prof,b) \wedge b \in I^\mathcal{U}(\Prof,a)\} \cup \dots$ be the voters who are influenced by $a$ both \emph{directly} and \emph{indirectly}.  

Given our domain $\Dom=\{0,1,*\}$, we consider the following rules: the \emph{majority rule} (Maj) returns the alternative in the domain having more than $\nicefrac{n}{2}$ votes, and $*$ otherwise; the \emph{relative majority rule} (RMaj) returns the plurality outcome in $\Dom \setminus \{*\}$, and if there is a tie it returns $*$. 
A voting rule $r$ on the domain $\{0,1,*\}^n$ satisfies \emph{monotonicity} if for any profile $X$, if $r(X)=x$ with $x\in\{0,1\}$ then $r(X_{+x})=x$, where $X_{+x}$ is obtained from $X$ by having one voter switch from either an initial vote of $1-x$ to $x$ or $*$, or from an initial vote of $*$ to $x$. 
Observe that both Maj and RMaj satisfy monotonicity. Due to this definition we can now show the following:\footnote{Note that Definition \ref{def:CastPart} slightly differs from the one given in previous work \cite{colley2020smart}, and thus Theorem~\ref{thm:cast} does not hold for $\RU$ or $\DRU$: a counterexample can be constructed exploiting the fact that an agent may prefer the outcome of one random iteration of the procedure to another.
}

\begin{theorem}\label{thm:cast}
Any monotonic rule $r$ with \textsc{Unravel($\mathbf{\#}$)} for $\# \in \{\mathbf{U}$, $\mathbf{DU}\}$ satisfies cast-participation for $ \Liqlang_*$ with domain $\Dom = \{0,1,*\}$.
\end{theorem}

\begin{proof}
Without loss of generality, assume that for agent $a\in\N$ we have $B_a = (1)$. To falsify cast-participation, we need to find a profile $\Prof$ with $r(\textsc{Unravel}(\mathbf{\#})(\Prof))=0$ or $*$, and a ballot $B'_a$ such that $r(\textsc{Unravel}(\mathbf{\#})(\Prof_{-a},B_a'))=1$, for $\# \in \{\mathbf{U}$, $\mathbf{DU}\}$.

First, observe that all voters $c\in I^\#_*(a,\Prof)$ vote for 1 in $\Prof$, since the language is restricted to single-agent delegations. 
Now, if $B_a'=0$ or $*$ (or they delegate to some agent who is assigned these votes), then by monotonicity the result of $\Prof'$ will keep being 0 or $*$. 
Moreover, all $c\not\in I^\#_*(a,\Prof)$ do not change their vote from $\Prof$ to $\Prof'$, no matter if $B_a'$ is a direct vote or a possibly ranked delegation. 
Therefore, the final votes of $\Prof'$ can be obtained from those of $\Prof$ by switching $1$s to $0$s or $*$s.
Thus, this contradicts the monotonicity assumption of rule $r$.
\end{proof} 

\begin{remark}
Theorem~\ref{thm:cast} does not hold for $\Blang$ ballots. Consider the counterexample with agents $\N=\{a, b, c\}$ having ballots $B_a= (1)$, $B_b=((\{a\},\neg a)>0)$ and $B_c=((\{a\},\neg a)>0)$. Each of our greedy unravellings would return the outcome $Maj(1,0,0)=0$. However, if $B_a'=0$ then $Maj(0,1,1)=1$. Thus, agent $a$ strictly prefers to submit a ballot that is not a direct vote for their preferred alternative. 
\end{remark}

We now focus on the incentive that a voter has to receive and accept delegations; namely, what has been introduced as the guru-participation property. 

\begin{definition}[Guru-participation]
A voting rule $r$ and a resolute unravelling procedure $\mathcal{U}$ satisfy \emph{guru-participation} if and only if for all profiles $\Prof$ and all agents $a \in\N$ such that $B_a=(x)$ with $x \in \Dom\setminus\{*\}$ we have that for any $b \in I^\#_*(\Prof, a)$
  $$r(\mathcal{U}(\Prof)) \geq _a  r(\mathcal{U}(\Prof_{-b},(*)))$$
where $\Prof_{-b}$ is $\Prof$ without $b$'s ballot. For randomised procedures we require the inequality to hold for any possible outcome of $\mathcal{U}$.
\end{definition}

All four greedy unravellings do not satisfy this property for the rule RMaj:

\begin{theorem}
$\text{\it RMaj}$ and $\textsc{Unravel}(\mathbf{\#})$ for $\mathbf{\#} \in \{\mathbf{U}$, $\mathbf{DU}$, $\mathbf{RU}$, $\mathbf{DRU}\}$ do not satisfy guru-participation for $\Liqlang_*$ with domain $\Dom = \{0,1,*\}$.
\end{theorem}

\begin{proof}
Consider a smart profile $\Prof$, as shown on the left hand-side of Table~\ref{table:notguru}, 
and profile  $\Prof'=(\Prof_{-b}, (*))$ obtained from $\Prof$ by switching $b$'s vote to $B_b' = (*)$. The outcomes of the four procedures are shown in the right hand-side of Table~\ref{table:notguru}.
\begin{table}[h]
\centering
\resizebox{0.4\textwidth}{!}{
  \begin{tabular}{llll} 
 \hline\noalign{\smallskip}
&$B_x^1$&$B_x^2$&$B_x^3$\\
 \noalign{\smallskip}\hline\noalign{\smallskip} 
$a$&$1$&-&-\\
$b$&$(\{c\}, id)$&$(\{a\},id)$&$*$\\
$c$&$(\{d\}, id)$&$(\{f\},id)$&$*$\\
$d$&$(\{b\}, id)$&$(\{f\},id)$&$*$\\
$e$&$1$&-&-\\
$f$&$0$&-&-\\
\noalign{\smallskip}\hline
 \end{tabular}}
 \quad
 \resizebox{0.55\textwidth}{!}{
  \begin{tabular}{ccc}  
   \hline\noalign{\smallskip}
    $\#$& $\Prof$& $\Prof'$\\
   \noalign{\medskip}\hline\noalign{\medskip}
    $\mathbf{U}$/    &$X^1=(1,1,0,0,1,0)$  &$X^{2}=(1,*, *, *, 1, 0)$  \\
     $\mathbf{DU}$&&\\
      \noalign{\medskip}\hline\noalign{\medskip}
    $\mathbf{RU}$/  &$X^3={(1,1,1,1,1,0)}$  & $X^{2}=(1,*, *, *, 1, 0)$  \\
    $\mathbf{DRU}$  &$X^4=(1,0,0,0,1,0)$  &                         \\
    &$X^5=(1,0,0,0,1,0)$  &                         \\ 
   \noalign{\medskip}\hline
  \end{tabular}}
  \caption{A profile $\Prof$ (on the left) and the outcomes of $\unravel{\U},\unravel{\DU},\unravel{\RU}$ and $\unravel{\DRU}$ on the profiles $\Prof$ and $\Prof'$  (on the right), where $\Prof'=(\Prof_{-b}, (*))$ is obtained from $\Prof$ by switching $b$'s vote to $B_b' = (*)$. }
  \label{table:notguru}
\end{table}

%
By applying $\unravel{\mathbf{U}}$ and $\unravel{\mathbf{DU}}$, agent $a$ prefers the outcome of $\Prof'$ to that of $\Prof$, since $\text{RMaj}(X^1)= *$ and $\text{RMaj}(X^2)= 1$. For $\unravel{\mathbf{RU}}$ and $\unravel{\mathbf{DRU}}$, the outcome on $\Prof'$ is $\text{RMaj}(X^2)=1$. However, the outcome on $\Prof$ can be either $\text{RMaj}(X^4)=\text{RMaj}(X^5)=0$ or $\text{RMaj}(X^3)=1$. Hence, when the random choice of $\RU$ or $\DRU$ leads to $X^4$ or $X^5$, agent $a$ strictly prefers the outcome $\text{RMaj}(X^2)$ to the outcome $\text{RMaj}(X^4)$ and $\text{RMaj}(X^5)$. Therefore, the inequality does not hold for any  outcome of the randomised procedures. 
\end{proof}


\subsection{Pareto dominance and optimality}\label{subsec:pareto}
We now focus on comparing the  outcomes of our unravelling procedures in terms of \emph{Pareto dominance} and \emph{Pareto optimality}. We show that none of our procedures always Pareto dominates another. However, we prove that all outcomes of $\minsum$ are Pareto optimal with respect to all outcomes with consistent certificates.

A certificate $\cert$ \emph{Pareto dominates} another certificate $\cert'$ if for every $i\in \N$, we have that $\cert_i\leq \cert'_i$, and there exists a $j\in\N$ such that $\cert_j < \cert'_j$. We say that the unravelling procedure $\mathcal{U}$ Pareto dominates another unravelling procedure $\mathcal{U}'$ when for any valid profile $\Prof$, all (possible) certificates $\cert$ obtained from  $\mathcal{U}(\Prof)$ Pareto dominate all the (possible) certificates $\cert$ obtained from  $\mathcal{U}'(\Prof)$. Note that the possibility of multiple certificates arises not only for irresolute procedures but also from different executions of the random procedures $\unravel{\RU}$ and $\unravel{\DRU}$.

\begin{example}\label{ex:Paretodom}
    Consider the example given in Table~\ref{table:UnravellingdifferentOutcomes}. The certificate of the outcome of $\unravel{\U}$ is $\cert^\U=(3,3,3,2)$, and that of the outcome of $\unravel{\DU}$ is $\cert^\DU=(1,3,3,2)$. Thus,  $\unravel{\DU}$ Pareto dominates $\unravel{\U}$, since $\cert^\DU$ Pareto dominates $\cert^\U$, given that each entry of $\cert^\DU$ is less than or equal to the corresponding entry in $\cert^\U$. 
Moreover, since there is an outcome of $\unravel{\RU}$ with certificate $\cert^\RU=(3,1,1,2)$, neither $\cert^\DU$ Pareto dominates $\cert^\RU$ (as $\cert^\DU_a< \cert^\RU_a$ for the first agent $a$) nor vice-versa (as $\cert_b^\DU>  \cert_b^\RU$ for the second agent $b$).
\end{example}


When comparing our greedy procedures, one might think that $\unravel{\DRU}$ should always be chosen, given that it has both direct vote priority and random voter selection. However, the following example shows a profile where $\unravel{\DU}$, $\unravel{\RU}$ and $\unravel{\DRU}$ do not Pareto dominate $\unravel{\U}$, and thus, they do not Pareto dominate $\unravel{\U}$ in general.

\begin{example}\label{ex:UnotParetodombygreedy}
    Take agents $\N=\{a,b,c,d,e,f\}$, whose ballots are shown in Table~\ref{table:exampleUnotPD}.
On this profile, $\unravel{\U}$ gives the outcome $X_\cert=(1,1,1,1,1,0)$, where $\cert=(1,1,1,1,2,2)$ and  $\unravel{\DU}$ gives $X_{\cert'}=(0,0,0,1,0,0)$, with certificate $\cert'=(3,3,3,1,1,2)$. Thus, $\cert$ does not Pareto dominate $\cert'$ as agent $e$ has that $\cert_{e} < \cert'_{e}$. It is also not the case that $\cert'$ Pareto dominates $\cert$ as for some agents, i.e., agent $a$, we have that $\cert_{a}>  \cert'_{a}$. Thus $\unravel{\DU}$ does not Pareto dominate $\unravel{\U}$ or vice-versa.

Furthermore, a possible outcome of $\unravel{\RU}$  is $X_{\cert'' }=(0,0,0,1,0,0)$ where $\cert'' =(3,1,1,1,1,2)$---the random choices picking first $f$ and then $a$. Again, $\cert''$ does not Pareto dominate $\cert$, as $\cert_{a}>  \cert''_{a}$. Therefore, $\unravel{\RU}$ does not Pareto dominate $\unravel{\U}$, and as $X_{\cert'' }$ is also an outcome of $\unravel{\DRU}$, $\unravel{\DRU}$ does not Pareto dominate $\unravel{\U}$ as well. 
\end{example}

    \begin{table}[h]
\centering
 \begin{tabular}{llll} 
 \hline\noalign{\smallskip}
&$B_x^1$&$B_x^2$&$B_x^3$\\
 \noalign{\smallskip}\hline\noalign{\smallskip} 
$a$&$(\{b,e\}, b\vee e)$&$(\{c,e\},c\vee e)$&$0$\\
$b$ & $(\{c,e\}, c\vee e)$&$(\{a,e\},a\vee e)$&$0$\\
$c$&$(\{a,e\},a\vee e)$&$(\{b,e\},b\vee e)$&$0$\\
$d$&$1$&-&-\\
$e$ &$(\{f\},f)$&$(\{d\},d)$&$0$\\
$f$&$(\{e\},e)$&$0$&-\\
\noalign{\smallskip}\hline
\end{tabular}
  \caption{A profile $\Prof$ showing that $\unravel{\U}$ is not dominated in general by $\unravel{\DU}$, $\unravel{\RU}$ or $\unravel{\DRU}$.}
  \label{table:exampleUnotPD}
\end{table}

\begin{proposition}\label{prop:greedyparetodom}
None of the four greedy unravelling procedures $\unravel{\#}$ for $\#\in\{\U,\DU,\RU,\DRU\}$ Pareto dominates another greedy procedure. 
\end{proposition}
\begin{proof}
Example~\ref{ex:UnotParetodombygreedy} shows that $\unravel{\U}$ is not Pareto dominated by $\unravel{\#}$ for $\#\in\{\DU,\RU,\DRU\}$. Then,  in Table~\ref{table:UnravellingdifferentOutcomes} the outcome of $\unravel{\U}$ is Pareto dominated by the outcomes of $\unravel{\#}$ for $\#\in\{\DU,\RU,\DRU\}$ and therefore, $\unravel{\U}$ does not Pareto dominate the other greedy procedures. 

From Table~\ref{table:UnravellingdifferentOutcomes}, we can also conclude that $\unravel{\DRU}$ does not Pareto dominate $\unravel{\DU}$, and vice-versa. The outcome of $\unravel{\DRU}$ with certificate $\cert=(2,1,3,2)$ does not Pareto dominate the outcome of $\unravel{\DU}$, having certificate $\cert'=(1,3,3,2)$, as $\cert_a>\cert_a'$. For the other direction, as $\cert'_b>\cert_b$ $\unravel{\DU}$ does not always Pareto dominate $\unravel{\DRU}$. Since the outcome with $\cert=(2,1,3,2)$ is also possible for $\unravel{\RU}$, $\unravel{\RU}$ is not guaranteed to Pareto dominate $\unravel{\DU}$ and vice-versa. 

Finally, $\unravel{\DRU}$ and $\unravel{\RU}$ do not Pareto dominate one another as the certificates of the former are a subset of the latter (Proposition~\ref{prop:DRUsubsetRU}).
\end{proof}

For irresolute procedures, by checking whether for any profile an unravelling procedure always has an outcome whose certificate Pareto dominates the certificates of all the outcomes of another procedure, we find the following negative results:
\begin{itemize}
    \item 
    The certificates $\cert=(4, 1, \dots , 1)$ of $\minsum$ and  $\cert'=(1,2,\dots, 2)$ of $\minmax$ from Example~\ref{ex:MinMaxbenefit} shot that neither $\cert$ Pareto dominates $\cert'$ nor vice-versa. Therefore, neither $\minmax$ nor $\minsum$ dominates the other.
    
    \item From Table~\ref{table:UnravellingdifferentOutcomes}, we see that $\unravel{\U}$ does not Pareto dominate $\minsum$ or $\minmax$ in general. From Example~\ref{ex:MinMaxbenefit} we see that $\minsum$ does not Pareto dominate $\unravel{\U}$ in general. 

\end{itemize}

Finally, we introduce the notion of \emph{Pareto optimality}, which defines all those consistent certificates that are not Pareto dominated by any other consistent certificate.


\begin{definition}\label{def:paretoopt}
A certificate $\cert$ for $\Prof$ is \emph{Pareto optimal} for the consistent certificates $\mathcal{C}(\Prof)$ if there exists no $\cert'\in\mathcal{C}(\Prof)$ with $\cert'\neq \cert$, such that $\cert'$ Pareto dominates $\cert$. 
\end{definition}
The following proposition corresponds to the well-known fact that maximising the average of a vector leads to a Pareto optimal vector, but not vice-versa.


\begin{proposition}\label{prop:minsumparetoopt}
The certificate $\cert$ for any outcome $X_\cert \in \minsum(\Prof)$ is Pareto optimal for $\mathcal{C}(\Prof)$, for any profile $\Prof$. 
\end{proposition}
\begin{proof}
Take an arbitrary valid smart profile $\Prof$, and arbitrary $X_\cert\in\minsum(\Prof)$. For the sake of a contradiction, assume that $\cert$ is not Pareto optimal for $\mathcal{C}(\Prof)$. Hence, there exists a $\cert'\in \mathcal{C}(\Prof)\setminus \{\cert\}$ such that $\cert'$ Pareto dominates $\cert$. Therefore, for all $i\in \N$, we get $\cert'_i\leq \cert_i$. We thus have that $\sum\limits_{i\in\N}\cert'_i\leq \sum\limits_{i\in\N} \cert_i$. 
Furthermore, since $\cert\neq \cert'$ and $\cert'$ Pareto dominates $\cert$, there exists an agent $j \in \N $ such that $\cert_j'< \cert_j$, and thus $\sum\limits_{i\in\N}\cert'_i< \sum\limits_{i\in\N} \cert_i$. Since $\sum\limits_{i\in\N} \cert_i$ is not minimal, we have $X_\cert\notin\minsum(\Prof)$, and thus we have reached a contradiction. 
\end{proof}
Note that the opposite direction does not hold: in Example~\ref{ex:MinMaxbenefit}, the $\minsum$ procedure does not return the Pareto optimal certificate $\cert=(1,2,\dots,2)$. Moreover, Proposition~\ref{prop:minsumparetoopt}  does not hold for the other unravelling procedures, as there exist outcomes of each of them whose certificates are Pareto dominated by some other consistent certificate, as seen in previous examples. 

\subsection{Discussion on unravelling procedures}

In this paper we have provided six unravelling procedures and we have given results that should guide a user of this model as to which procedure to choose. Here we provide a summary and a discussion of these results. 

The main distinction between the optimal and greedy procedures is that finding an outcome with a greedy procedure is a tractable problem, whereas even checking if an outcome of an optimal procedure exists under a given bound on the optimised score is an $\textsc{NP}$-complete problem for the general language $\Blang$ (where the delegations are contingent formulas expressed in complete DNF). Although we acknowledge that the improvements in performance of \textsc{Sat}-solvers make the intractability of the problems $\minmaxbound$ and $\minsumbound$ less concerning, the associated search problem of computing the outcomes of the unravelling remains in principle even harder. Hence, the greedy procedures are desirable when tractability is key.

Proposition~\ref{prop:minsumparetoopt} shows that the certificates of the outcomes of $\minsum$ are Pareto optimal and thus are never dominated by outcomes found by a consistent certificate. In contrast, $\minmax$ cannot make this guarantee. Although $\minmax$ may return outcomes that are not Pareto optimal, it provides more egalitarian outcomes.
In Example~\ref{ex:MinMaxbenefit}, the outcome with the lowest $\rank$ relies on the fourth preference of agent~$a$ being chosen: while it is still a trusted delegate, the agent may be less confident in them than in the three previous delegations. 

Furthermore, as $\minsum$ and $\minmax$ are irresolute, they would have to be paired with a tie-breaking mechanism to select a single outcome from the possibly many that they produce. In contrast, the greedy procedures are not only, in general, quicker than the optimal procedures, but they are also resolute.

With profiles of $\liqlang{1}{*}$ ballots, $\minsum$  and the greedy procedures return the same outcome vector, so they can be used interchangeably. For profiles of $\Liqlang$ ballots, there should be a preference for $\minsum$ or $\minmax$, since an outcome can be found in polynomial time (Theorems~\ref{thm:optcomplex} and~\ref{thm:minmaxliquid}). The choice between these two procedures should then be determined by whether the situation would benefit more from Pareto optimality or egalitarian properties. However, these procedures do rely on tie-breaking, which could bring up issues of fairness in the certificates. 

As the participation axioms do not differentiate the greedy procedures, the properties that they are defined on (i.e., direct vote priority and random voter selection) are the clearest way to compare them. Random voter selection should be used when a lottery is acceptable and it should be avoided when it would be unfair to give a worse preference level to just some agents.
Direct vote priority should be used when a direct vote from an agent is preferred to a delegation, perhaps in situations that could benefit from a level of expertise on the issue, and to ensure shorter delegation chains.

 Given the above discussion, one may think that $\unravel{\DRU}$ gives the best outcomes overall. However, we have proved that no greedy procedure is guaranteed to Pareto dominate another (Proposition~\ref{prop:greedyparetodom}). Thus, the notion of Pareto dominance does not distinguish between greedy procedures.

In summary, greedy procedures should be preferred when finding outcomes needs to be tractable, except in the special case of $\Liqlang$ ballots for which this problem is polynomial for all proposed rules. The $\minsum$ procedure should be used when outcomes need to be Pareto optimal and $\minmax$ should be used when an egalitarian approach is required. When using the greedy procedures, the choice between them should be determined by whether the situation asks for either of the random voter selection and direct vote priority properties.

\section{Conclusion}\label{sec:conc}

We proposed a model of multi-agent ranked delegations in voting, which generalises liquid democracy in two aspects. The first is that delegations are more expressive, as they can involve many agents instead of a single one, who in turn determine their vote. We introduced a general language named $\Blang$, in which delegations are expressed as contingent propositional formulas in complete DNF. We emphasise that although agents may not want to use the full expressivity of the language, many natural delegations types are captured by it: for example, both liquid democracy delegations and delegations using threshold rules can be expressed in $\Blang$. 
Our second generalisation is the possibility of ranked delegations: as transitive delegations can lead to cycles among the agents' most preferred delegates, the linear order of trusted delegations given by the agents can be used to break these cycles. 

Our main contribution is the definition and study of six unravelling procedures (two optimal and four greedy ones), which take a profile of smart ballots and return a standard voting profile. We show that all of the procedures can give different certificates and outcomes (Proposition~\ref{prop:unravelldifferent}), and that they differ from the breadth-first and depth-first procedures of \citet{kotsialou2018breadth} (Remark~\ref{remark:breadth+depth}).
Moreover, we show that all of the procedures, except $\minmax$, coincide on $\liqlang{1}{}$ ballots, i.e., liquid democracy ballots with a single delegation per agent. The certificates of the outcomes of $\minsum$ are Pareto optimal with respect to the outcomes found from any consistent certificate (Proposition~\ref{prop:minsumparetoopt}), while greedy procedures do not Pareto dominate one another (Proposition~\ref{prop:greedyparetodom}).
Our main results are that deciding if there exists an outcome of the unravellings $\minsum$ and $\minmax$ bounded by some constant 
are \textsc{NP}-complete problems (Theorems~\ref{thm:minsumNPcomp} and~\ref{thm:minmaxNP}) 
over the general language $\Blang$, but they become tractable
when ballots are restricted to $\Liqlang$, the language of ranked liquid democracy (Theorems~\ref{thm:optcomplex} and~\ref{thm:minmaxliquid}). 
Finally, we prove that our four greedy unravelling procedures always terminate (Proposition~\ref{prop:unraveltemrination}) and do so in a polynomial number of time steps for general $\Blang$ ballots (Proposition~\ref{prop:polynomialterminationunravel}).

\paragraph{Future work.} This paper provides a first analysis of six unravelling procedures in terms of their computational (and some axiomatic) properties. 
A game-theoretic analysis of our procedures is yet to be undertaken, by focusing, e.g., on various notions of manipulative actions. 
Moreover, we have shown that finding a solution for our optimal unravellings is a tractable problem when ballots are restricted to ranked single agent delegations. Other tractable cases could be found, e.g., by restricting delegation functions or by limiting the number of delegations---leading to a study of the parameterised complexity of optimal unravellings. 
Finally, we focused only on independent issues: extending our model to account for the agents' rationality with respect to interconnected issues (in line with the work of \citet{christoff2017binary} and, to a lesser extent, \citet{BrillTalmonIJCAI2018}) would be a natural avenue of future research.

\section*{Acknowledgments}
The authors acknowledge the support of the ANR JCJC
project SCONE (ANR 18-CE23-0009-01).

%


%
%


%
%

\printbibliography

@Article{Green-Armytage2015,
author="Green-Armytage, James",
title="Direct voting and proxy voting",
journal="Constitutional Political Economy",
year="2015",
volume="26",
number="2",
pages="190--220"
}

@misc{BrillPresentation,
author = {Markus Brill and Anne-Marie George and Martin Lackner and Ulrike Schmidt-Kraepelin},
title = {Liquid Democracy with Ranked Delegations},
year= {2021},
note = {Presentation at the Workshop on Liquid Democracy at the University of Toulouse},
url={www.irit.fr/~Umberto.Grandi/scone/WK_Ulrike.pdf}
}

@book{Moulin1988, 
title={Axioms of Cooperative Decision Making}, 
publisher={Cambridge University Press}, 
author={Moulin, Herv{\'e}}, 
year={1988}
}

@inproceedings{colley2020smart,
  title={Smart Voting},
  author={Colley, Rachael and Grandi, Umberto and Novaro, Arianna},
  booktitle={Proceeding of the the 29th International Joint Conference on Artificial Intelligence (IJCAI)},
  year={2020}
}

@techreport{HardtLopes,
author = {Hardt, Steve and Lopes, Lia C. R.},
title = {Google Votes: A Liquid Democracy Experiment on a Corporate Social Network},
year= {2015},
institution = {Technical Disclosure Commons, Google}
}

@inproceedings{AbramowitzM19,
  author    = {Ben Abramowitz and
               Nicholas Mattei},
  title     = {Flexible Representative Democracy: An Introduction with Binary Issues},
  booktitle = {Proceedings of the 28th International Joint Conference on
               Artificial Intelligence (IJCAI)},
  year      = {2019}
}

@Book{BehrensEtAl2014,
  author =       {Jan Behrens and Axel Kistner and Andreas Nitsche and Bj{\"o}rn Swierczek},
  title =        {Principles of Liquid Feedback},
  publisher =    {Interacktive Demokratie},
  year =         {2014},
}

@book{crama2011boolean,
  title={Boolean functions: Theory, algorithms, and applications},
  author={Crama, Yves and Hammer, Peter L},
  year={2011},
  publisher={Cambridge University Press}
}

@incollection{karp1972reducibility,
  title={Reducibility among combinatorial problems},
  author={Karp, Richard M},
  booktitle={Complexity of computer computations},
  pages={85--103},
  year={1972},
  publisher={Springer}
}

@article{behrens2015finite,
  title={Preferential Delegation and the Problem of Negative Voting Weight},
  author={Behrens, Jan and Swierczek, Bj{\"o}rn},
  journal={The Liquid Democracy Journal},
  volume={3},
  year={2015}
}

@inproceedings{BrillTalmonIJCAI2018,
  author    = {Markus Brill and
               Nimrod Talmon},
  title     = {Pairwise Liquid Democracy},
  booktitle = {Proceedings of the 27th International Joint Conference on
               Artificial Intelligence (IJCAI)},
  year      = {2018}
}

@inproceedings{CohensiusEtAlAAMAS2017,
  author    = {Gal Cohensius and
               Shie Mannor and
               Reshef Meir and
               Eli A. Meirom and
               Ariel Orda},
  title     = {Proxy Voting for Better Outcomes},
  booktitle = {Proceedings of the 16th Conference on Autonomous Agents and MultiAgent
               Systems (AAMAS)},
  year      = {2017}
}

@Inproceedings{christoff2017binary,
  author    = {Christoff, Zo\'e and Grossi, Davide},
  year      = {2017},
  title     = {Binary Voting with Delegable Proxy: An Analysis of Liquid Democracy},
  booktitle = {Proceedings of the 16th Conference on Theoretical Aspects of Rationality and Knowledge (TARK)}
}

@inproceedings{brill2018interactive,
  title={Interactive democracy},
  author={Brill, Markus},
  booktitle={Proceedings of the 17th International Conference on Autonomous Agents and MultiAgent Systems (AAMAS)},
  year={2018}
}

@inproceedings{BGL19,
  author    = {Daan Bloembergen and
               Davide Grossi and
               Martin Lackner},
  title     = {On Rational Delegations in Liquid Democracy},
  booktitle = {Proceedings of the 33rd {AAAI} Conference on Artificial Intelligence (AAAI)},
  year      = {2019}
}

@inproceedings{kotsialou2018breadth,
title = "Incentivising Participation in Liquid Democracy with Breadth-First Delegation",
author = "Grammateia Kotsialou and Luke Riley",
year = "2020",
booktitle = "Proceedings of the 19th International Conference on Autonomous Agents and Multi-Agent Systems (AAMAS)",
}

@inproceedings{golz2018fluidliquid,
  title={The fluid mechanics of liquid democracy},
  author={G{\"o}lz, Paul and Kahng, Anson and Mackenzie, Simon and Procaccia, Ariel D},
  booktitle={International Conference on Web and Internet Economics (ICWIE)},
  year={2018}
}

@article{miller1969program,
  title={A program for direct and proxy voting in the legislative process},
  author={Miller, James C},
  journal={Public choice},
  volume={7},
  number={1},
  pages={107--113},
  year={1969},
  publisher={Springer}
}

@article{degrave2014resolving,
  title={Resolving multi-proxy transitive vote delegation},
  author={Degrave, Jonas},
  journal={arXiv preprint arXiv:1412.4039},
  year={2014}
}

@inproceedings{kahng2018liquid,
  title={Liquid democracy: An algorithmic perspective},
  author={Kahng, Anson and Mackenzie, Simon and Procaccia, Ariel D},
  booktitle={Proceedings of the 32nd AAAI Conference on Artificial Intelligence (AAAI)},
  year={2018}
}

@inproceedings{caragiannis2019contribution,
  title={A contribution to the critique of liquid democracy},
  author={Caragiannis, Ioannis and Micha, Evi},
  booktitle={Proceedings of the 28th International Joint Conference on Artificial Intelligence (IJCAI)},
  year={2019},
}

@techreport{dhillon2019introduction,
  title={Introduction to Voting and the Blockchain: some open questions for economists},
  author={Dhillon, Amrita and Kotsialou, Grammateia and McBurney, Peter and Riley, Luke and others},
  year={2019},
  institution={Competitive Advantage in the Global Economy (CAGE)}
}

@inproceedings{EscoffierEtAlSAGT2019,
  author    = {Bruno Escoffier and
               Hugo Gilbert and
               Ad{\`{e}}le Pass-Lanneau},
  title     = {The Convergence of Iterative Delegations in Liquid Democracy in a
               Social Network},
  booktitle = {Proceedings of the 12th International Symposium on Algorithmic Game Theory (SAGT)},
  year      = {2019}
}

@inproceedings{konczak2005voting,
  title={Voting procedures with incomplete preferences},
  author={Konczak, Kathrin and Lang, J{\'e}r{\^o}me},
  booktitle={Proceedings of the IJCAI-05 Multidisciplinary Workshop on Advances in Preference Handling {(MPREF)}},
  year={2005}
}

@article{edmonds1967optimum,
  title={Optimum branchings},
  author={Edmonds, Jack},
  journal={Journal of Research of the national Bureau of Standards B},
  volume={71},
  number={4},
  pages={233--240},
  year={1967}
}

@book{cormen2009introduction,
  title={Introduction to algorithms},
  author={Cormen, Thomas H and Leiserson, Charles E and Rivest, Ronald L and Stein, Clifford},
  year={2009},
  publisher={MIT press}
}

@book{kozen2012design,
  title={The design and analysis of algorithms},
  author={Kozen, Dexter C},
  year={2012},
  publisher={Springer Science \& Business Media}
}

@article{chu1965shortest,
  title={On the shortest arborescence of a directed graph},
  author={Chu, Yoeng-Jin},
  journal={Scientia Sinica},
  volume={14},
  pages={1396--1400},
  year={1965}
}

@article{bock1971algorithm,
  title={An algorithm to construct a minimum directed spanning tree in a directed network},
  author={Bock, F. C. },
  journal={Developments in operations research},
  pages={29--44},
  year={1971},
  publisher={Gordon and Breach}
}

@article{shapiro2018point,
  title={Point: foundations of e-democracy},
  author={Shapiro, Ehud},
  journal={Communications of the ACM},
  volume={61},
  number={8},
  pages={31--34},
  year={2018},
  publisher={ACM New York, NY, USA}
}

@book{tullock1967toward,
  title={Toward a mathematics of politics},
  author={Tullock, Gordon},
  year={1967},
  publisher={Ann Arbor: University of Michigan Press}
}

@article{mueller1972representative,
  title={Representative democracy via random selection},
  author={Mueller, Dennis C and Tollison, Robert D and Willett, Thomas D},
  journal={Public Choice},
  volume={12},
  number={1},
  pages={57--68},
  year={1972},
  publisher={Springer}
}

@article{alger2006voting,
  title={Voting by proxy},
  author={Alger, Dan},
  journal={Public Choice},
  volume={126},
  number={1-2},
  pages={1--26},
  year={2006},
  publisher={Springer}
}

@article{lanphier1995model,
  title={A model for electronic democracy},
  author={Lanphier, Rob},
  journal={Manuscript. http://robla. net/1996/steward},
  year={1995}
}

@techreport{ford2002delegative,
  title={Delegative democracy},
  author={Ford, Bryan Alexander},
  year={2002}
}

@article{tullock1992computerizing,
  title={Computerizing politics},
  author={Tullock, Gordon},
  journal={Mathematical and Computer Modelling},
  volume={16},
  number={8-9},
  pages={59--65},
  year={1992},
  publisher={Elsevier}
}

@inproceedings{litvinenko2012social,
  title={Social media and perspectives of liquid democracy on the example of political communication of Pirate Party in Germany},
  author={Litvinenko, Anna},
  booktitle={Proceedings of the 12th European Conference on e-Government in Barcelona},
  year={2012}
}

@article{blum2016liquid,
  title={Liquid democracy: Potentials, problems, and perspectives},
  author={Blum, Christian and Zuber, Christina Isabel},
  journal={Journal of Political Philosophy},
  volume={24},
  number={2},
  pages={162--182},
  year={2016},
  publisher={Wiley Online Library}
}

@article{swierczek2014five,
  title={Five years of Liquid Democracy in Germany},
  author={Swierczek, Bj{\"o}rn},
  journal={The Liquid Democracy Journal},
  volume={1},
  number={1},
  pages={8--19},
  year={2014}
}

@article{boldi2011viscous,
  title={Viscous democracy for social networks},
  author={Boldi, Paolo and Bonchi, Francesco and Castillo, Carlos and Vigna, Sebastiano},
  journal={Communications of the ACM},
  volume={54},
  number={6},
  pages={129--137},
  year={2011},
  publisher={ACM New York, NY, USA}
}

@inproceedings{zhang2019statement,
  title={Statement voting},
  author={Zhang, Bingsheng and Zhou, Hong-Sheng},
  booktitle={International Conference on Financial Cryptography and Data Security},
  year={2019}
}

@inproceedings{cohensiusproxy,
  title={Proxy Voting for Revealing Ground Truth},
  author={Cohensius, Gal and Meir, Reshef},
  year={2017},
  booktitle={Proceedings of the 4th Workshop on Exploring Beyond the Worst Case in Computational Social Choice (EXPLORE)}
}

@book{dodgson1885principles,
  title={The Principles of Parliamentary Representation},
  author={Dodgson, Charles Lutwidge},
  year={1884},
  publisher={Harrison and Sons}
}

@inproceedings{zhang2020power,
  title={Power in Liquid Democracy},
  author={Zhang, Yuzhe and Grossi, Davide},
  booktitle={Proceedings of the 35th {AAAI} Conference on Artificial Intelligence {(AAAI)}},
  year={2021}
}

@article{mancini2015time,
  title={Why it is time to redesign our political system},
  author={Mancini, Pia},
  journal={European View},
  volume={14},
  number={1},
  pages={69--75},
  year={2015},
  publisher={SAGE Publications Sage UK: London, England}
}

@article{meir2020sybil,
  title={Sybil-Resilient Social Choice with Partial Participation},
  author={Meir, Reshef and Shahaf, Gal and Shapiro, Ehud and Talmon, Nimrod},
  journal={arXiv preprint arXiv:2001.05271},
  year={2020}
}

@inproceedings{kling2015voting,
  title={Voting behaviour and power in online democracy: A study of LiquidFeedback in Germany's Pirate Party},
  author={Kling, Christoph and Kunegis, J{\'e}r{\^o}me and Hartmann, Heinrich and Strohmaier, Markus and Staab, Steffen},
  booktitle={Proceedings of the International AAAI Conference on Web and Social Media},
  year={2015}
}

@inproceedings{escoffier2020iterative,
  title={Iterative delegations in liquid democracy with restricted preferences},
  author={Escoffier, Bruno and Gilbert, Hugo and Pass-Lanneau, Ad{\`e}le},
  booktitle={Proceedings of the 34th AAAI Conference on Artificial Intelligence ({AAAI})},
  year={2020}
}

@article{grandi2017social,
  title={Social choice and social networks},
  author={Grandi, Umberto},
  journal={Trends in Computational Social Choice},
  pages={169--184},
  year={2017},
}

@inproceedings{DBLP:conf/ijcai/BredereckE17,
  author    = {Robert Bredereck and               Edith Elkind},
  title     = {Manipulating Opinion Diffusion in Social Networks},
  booktitle = {Proceedings of the 26th International Joint Conference on
               Artificial Intelligence {(IJCAI)}},
  year      = {2017}
}

@inproceedings{DBLP:conf/ijcai/BrillEEG16,
  author    = {Markus Brill and Edith Elkind and Ulle Endriss and Umberto Grandi},
  title     = {Pairwise Diffusion of Preference Rankings in Social Networks},
  booktitle = {Proceedings of the 25th International Joint Conference on
               Artificial Intelligence {(IJCAI)}},
  year      = {2016}
}

@article{granovetter1978threshold,
  title={Threshold models of collective behavior},
  author={Granovetter, Mark},
  journal={American journal of sociology},
  volume={83},
  number={6},
  pages={1420--1443},
  year={1978},
  publisher={University of Chicago Press}
}

\end{document}